%% file: AAMAS_2026_sample.tex
\documentclass[sigconf]{aamas}

\usepackage{balance} %
\usepackage{latexsym}
\usepackage{amsmath}
\usepackage{amsthm}
\usepackage{booktabs}
\usepackage{enumitem}
\usepackage{graphicx}
\usepackage{color}
\usepackage{nicefrac}
\usepackage{subfigure}
\usepackage{subcaption}
\usepackage{multirow} %
\usepackage{array} %
\usepackage{listings} %
\usepackage{dsfont}

\newtheorem{remark}{Remark}
\newtheorem{proposition}{Proposition}
\newtheorem{definition}{Definition}

\setcopyright{ifaamas}
\acmConference[AAMAS '26]{Proc.\@ of the 25th International Conference
on Autonomous Agents and Multiagent Systems (AAMAS 2026)}{May 25 -- 29, 2026}
{Paphos, Cyprus}{C.~Amato, L.~Dennis, V.~Mascardi, J.~Thangarajah (eds.)}
\copyrightyear{2026}
\acmYear{2026}
\acmDOI{}
\acmPrice{}
\acmISBN{}

\acmSubmissionID{1422}

\title[AAMAS-2026 Formatting Instructions]{Do LLMs Strategically Reveal, Conceal, and Infer Information? \\ A Theoretical and Empirical Analysis in The Chameleon Game}

\author{Mustafa O. Karabag}
\affiliation{
  \institution{The University of Texas at Austin}
    \country{}}

\author{Jan Sobotka}
\affiliation{
  \institution{École Polytechnique Fédérale de Lausanne}
  \country{}}
  \authornote{This work was completed while Jan Sobotka was with The University of Texas at Austin.}

\author{Ufuk Topcu}
\affiliation{
  \institution{The University of Texas at Austin}
  \country{}}

\renewcommand\footnotetextcopyrightpermission[1]{} %
\begin{abstract}
\input{paper_sections/abstract}

\end{abstract}


\newcommand{\BibTeX}{\rm B\kern-.05em{\sc i\kern-.025em b}\kern-.08em\TeX}

\input{commands}

\newcommand{\ra}[1]{\renewcommand{\arraystretch}{#1}}
\usepackage{ifthen}

\newboolean{supversion}
\setboolean{supversion}{true} %

\newcommand{\supptext}[2]{%
  \ifthenelse{\boolean{supversion}}{#1}{#2}%
}

\begin{document}

\pagestyle{fancy}
\fancyhead{}

\maketitle 

\input{paper_sections/introduction}

\input{paper_sections/related_work}

\input{paper_sections/the_game}

\input{paper_sections/theory}

\input{paper_sections/experiments}

\input{paper_sections/conclusion}

\clearpage
\bibliographystyle{ACM-Reference-Format} 
\bibliography{mybibfile}

\clearpage
\onecolumn

\supptext{\input{paper_sections/appendix}}{}

\end{document}

%% file: paper_sections/abstract.tex
Large language model-based (LLM-based) agents have become common in settings that include non-cooperative parties. In such settings, agents' decision-making needs to conceal information from their adversaries, reveal information to their cooperators, and infer information to identify the other agents' characteristics. To investigate whether LLMs have these information control and decision-making capabilities, we make LLM agents play the language-based hidden-identity game, The Chameleon. In this game, a group of non-chameleon agents who do not know each other aim to identify the chameleon agent without revealing a secret. The game requires the aforementioned information control capabilities both as a chameleon and a non-chameleon. We begin with a theoretical analysis for a spectrum of strategies, from concealing to revealing, and provide bounds on the non-chameleons' winning probability. The empirical results with GPT, Gemini 2.5 Pro, Llama 3.1, and Qwen3 models show that while non-chameleon LLM agents identify the chameleon, they fail to conceal the secret from the chameleon, and their winning probability is far from the levels of even trivial strategies. Based on these empirical results and our theoretical analysis, we deduce that LLM-based agents may reveal excessive information to agents of unknown identities. Interestingly, we find that, when instructed to adopt an information-revealing level, this level is linearly encoded in the LLM's internal representations. While the instructions alone are often ineffective at making non-chameleon LLMs conceal, we show that steering the internal representations in this linear direction directly can reliably induce concealing behavior.

%% file: commands.tex
\newcommand{\numplayers}{N}
\newcommand{\category}{C}
\newcommand{\word}{w}
\newcommand{\words}{W}
\newcommand{\numwordspercategory}{K}
\newcommand{\order}{o}

\newcommand{\secretword}{\word^{*}}
\newcommand{\response}{r}
\newcommand{\responses}{R}
\newcommand{\chameleon}{c}
\newcommand{\distribution}{\mathcal{D}}

%% file: paper_sections/introduction.tex
\section{Introduction}

LLM-based agents have become common in domains including education~\cite{xiao2023evaluating,alier2023smart}, healthcare~\cite{karabacak2023embracing,singhal2023large}, finance~\cite{wu2023bloomberggpt,yang2023fingpt}, and software development~\cite{lin2024llm,dong2024self}. These agents need to operate in the presence of other autonomous agents and humans where the interactions are not necessarily cooperative, such as negotiation scenarios~\cite{abdelnabi2024cooperation,bianchi2024well}. In these interactions, LLMs need to strategically strike a balance between potentially \textit{conflicting objectives} of concealing information from their adversaries, revealing information to their potential cooperators, and inferring the characteristics of others.

To investigate whether LLMs have such information control and decision-making capabilities, we use a language-based hidden-identity board game, The Chameleon~\cite{chameleon}. We theoretically analyze the strategies for the game and investigate the performance of existing models, GPT-5~\cite{openai_gpt5}, GPT-4.1~\cite{openai_gpt41}, GPT-4o~\cite{openai_gpt4_o}, Gemini 2.5 Pro~\cite{gemini}, Llama 3.1 70B~\cite{llama3}, and Qwen3 32B~\cite{qwen3}\footnote{We use instruction-tuned and quantized versions of the open-weight models Llama 3.1 and Qwen3 for efficiency of the experiments.}.

The Chameleon is a zero-sum hidden-identity board game between a chameleon player and $N-1$ non-chameleon players. The non-chameleons jointly know a set of words and a secret word randomly chosen from this set, but they do not know each other. The chameleon, on the other hand, does not know the secret word but knows the set of potential secret words. The players, including the chameleon, respond with a single word in a randomly chosen order. Then, the players vote to identify the chameleon. If the chameleon is correctly identified, it is asked to guess the secret word in the second stage. The chameleon wins the game if it is not identified after voting or if it guesses the secret word correctly.

We theoretically analyze the game by considering that the non-chameleons can share a pre-agreed strategy. The chameleon knows the strategy but not the secret word. We first consider two families of stationary strategies, pairwise concealing and revealing strategies, that do not take the previous responses into account. A pairwise concealing strategy shares similar response distributions for each pair of words. We show that under such concealing strategies, the non-chameleons fail to catch the chameleon in the voting stage with high probability, i.e., $1- \mathcal{O}(\nicefrac{1}{\numplayers})$. A pairwise revealing strategy shares distinct response distributions for each pair of possible secret words. We show that under such revealing strategies, the non-chameleons lose the game with high probability, i.e., $1-\mathcal{O}(\exp(-\numplayers))$ in the second stage. We then consider a history-dependent strategy that eliminates some of the possible secret words in each response while preserving the maximum ambiguity over the remaining ones. We show that using such a revealing and concealing strategy, the non-chameleons win the game with probability at least $\mathcal{O}(\nicefrac{\log(\numplayers)}{\numplayers})$.

We instruct the aforementioned LLMs
to play the Chameleon with each other. %
The empirical win ratio of non-chameleons is far below the theoretically achievable level of $\mathcal{O}(\nicefrac{\log(N)}{N})$ and the trivially achievable level of $\mathcal{O}(\nicefrac{1}{N})$. For example, for 4-player games with 16 words per category, the empirical win ratio of the GPT-5 non-chameleons is $0\%$ against a GPT-5 chameleon, while the non-chameleons could trivially achieve $23\%$ if they were ignoring the secret word information. The results show that while the non-chameleons identify the chameleon successfully with high probability, the chameleon successfully infers the secret word from the non-chameleons' answers. For example, the identification rate of the GPT-5 non-chameleons is $64\%$ against a GPT-5 chameleon. 

Combined with the theoretical analysis of pairwise revealing strategies, we deduce that non-chameleon LLMs reveal excessive information as they lose the game in the second stage. Given the failures of LLMs in the game, we first instruct LLMs to adopt different levels of information concealment. In turn, our Principal Component Analysis of the LLM's internal representations shows that the information concealment level is clearly reflected in these hidden states\footnote{We use the terms \textit{internal representations} and \textit{hidden states} interchangeably.} as a linear direction. However, even at the most extreme level of instructed concealment, non-chameleons perform similarly to the nominal case. To achieve information concealment levels that are beyond what is achievable by instructions, we use the linear direction to design a (category-independent) steering vector of desired strength, which is added as a constant bias to hidden state values instead of instructions. We observe that this approach results in the chameleon guessing the secret with a significantly lower success rate and, thereby, improves the win rates of the non-chameleons.

We make the following contributions:
\begin{itemize}
\setlength\itemsep{0em}
    \item We provide a theoretical analysis of the language-based hidden-identity game, The Chameleon~\cite{chameleon}. We consider a spectrum of strategies, from revealing to concealing, and provide bounds on the achievable performance levels. 
    \item We evaluate the performance of various LLMs and show that non-chameleon LLMs lose with high probability as they use revealing strategies, indicating that LLMs may lack rationality due to their poor information-concealing skills.
    \item We show that when instructions are used, the information concealment level is linearly encoded in internal representations, but instructions alone are not effective. Using this linear direction, we deploy hidden state steering and improve performance by amplifying information concealment beyond what is achievable with instructions.
\end{itemize}

We note that while our theoretical and empirical results analyze the information control skills of LLMs in the Chameleon game, they offer broader insights on the LLMs' lack of suitability for strategic interactions involving informational asymmetries. 
Combining theoretical and empirical results, our key observations are that (i) LLMs may have poor game-theoretic rationality due to their limited capacity for information concealment, (ii) while prompt-based instructions are somewhat reflected in internal states, the induced behavioral changes for information concealment are not sufficient. Given the improved performance via steering, we hypothesize that LLMs have an internal understanding of how to conceal information. However, their realization of it is weak since they fail to assess how much their actions influence others.

%% file: paper_sections/related_work.tex
\section{Related Work}

\textbf{\textit{Strategic decision-making and LLMs:}}
Strategic decision-making and learning capabilities of LLMs have been studied using various models, including multi-armed bandits~\cite{krishnamurthy2024can,park2024llm} and canonical games from game theory~\cite{brookins2023playing,lore2023strategic,akata2023playing}. As LLMs often do not exhibit rational behavior in games, existing works focused on improving the performance of LLMs by developing agent architectures~\cite{park2023generative,yang2024reasoning,thudiumoutwit,guo2023suspicion}, fine-tuning LLMs~\cite{huang2024pokergpt}, and using game-theoretic solvers~\cite{gemp2024steering,hua2024game}. 

While most of the existing works focus on empirically analyzing and improving the performance of LLMs for games, we provide theoretical and empirical analyses for the hidden identity game The Chameleon. Our results show that the irrational behavior of LLMs in strategic interactions may stem from losing informational advantage as they share excessive information with other agents of unknown identities.

\textbf{\textit{Social deduction games and LLMs:}}
Conversation-based social deduction games, including hidden role/identity games, provide a natural medium to test the capabilities of LLMs. Existing works showed LLMs have promising capabilities in games such as Werewolf~\cite{xu2023exploring,xu2023language}, Diplomacy~\cite{meta2022human}, Welfare Diplomacy~\cite{mukobi2023welfare}, Hoodwinked~\cite{o2023hoodwinked}, The Chameleon~\cite{xu2024magic}, and Spyfall~\cite{kim2023generative}. These games rely on open-ended discussions and negotiations between the players, and existing studies focus on evaluating or improving the empirical performance of LLMs. 

The simple nature of The Chameleon allows us to develop a model of the game and theoretically analyze it. Different from the existing works, we identify the strengths and weaknesses of LLMs based on both theory and experiments. The work \cite{kim2023generative} made the observation that non-spy players rely too much on keywords in the game of Spyfall, providing hints to the spy. Our theoretical and empirical results for The Chameleon support this observation. The work \cite{xu2024magic} uses The Chameleon to empirically analyze the skills of LLMs, such as judgment, reasoning, and deception, using handcrafted evaluation metrics. Different from \cite{xu2024magic}, we theoretically analyze different strategies for this game, assess the information control skills of LLMs by relating their empirical behavior to the theory, and show that the poor performance of non-chameleon LLMs is due to excessive revealing of information.

\textbf{\textit{LLMs, lying, and deception:}}
Recent studies on LLMs show that LLMs can lie to deceive others~\cite{hagendorff2024deception,burger2024truth,hou2024large} and detect deceivers~\cite{yoo2024finding}. These works show that LLMs can deliberately mislead others given a context, e.g., lying to a thief about the location of a valuable item~\cite{hagendorff2024deception}. The Chameleon, on the other hand, does not require the players to explicitly lie in their responses to induce false beliefs. Instead, it requires them to be ambiguous to conceal the secret word or their identity: Providing a response to mislead the chameleon can make a non-chameleon get voted by the others, since it fails to reveal that it is not the chameleon. The chameleon, on the other hand, needs to provide ambiguous answers to conceal its lack of secret word knowledge. 

\textbf{\textit{Privacy, security, and LLMs:}}
LLMs exhibit various privacy and security vulnerabilities~\cite{yao2024survey, wang2024unique,das2024security}. Works on LLM privacy mainly consider whether LLMs can conceal information from adversaries who aim to actively extract secret information. On the other hand, in a non-zero-sum incomplete information game, the rational agents do not necessarily only conceal information, and they, on their own, may reveal information to the other parties to improve cooperation. 
Our work shows that in an incomplete information game setting, as a result of their decision-making, LLMs cannot strike a balance between these objectives and leak secret information without an external adversary.

\textbf{\textit{Controlling LLMs via hidden-state interventions:}}
Directly intervening on the internal states of LLMs to guide their behavior is a growing area of research known as activation or representation engineering~\cite{turner2023steering,zou2023representation}. A common technique involves creating \textit{steering vectors} by taking the difference between activations from contrasting prompts, which, when added to the model's residual stream, can shift its output towards (un)desired high-level properties such as sentiment, truthfulness, or linguistic style~\cite{wang2025adaptive,li2025fairsteer,zhang2025control,turner2023steering}. Beyond simple stylistic control, these techniques have been extended to modulate more complex cognitive behaviors, with studies showing that inference-time interventions can improve performance on reasoning tasks~\cite{hjer2025improving,chen2025seal}. While much of this work focuses on the general behavioral control of LLMs as helpful and harmless chat assistants, we focus on controlling LLMs as agents in competitive and strategic game scenarios where steering vectors change the agent's action distribution.

%% file: paper_sections/the_game.tex
\section{The Chameleon Game}
The Chameleon\footnote{We use `The Chameleon' to refer to the game and `the chameleon' to refer to the player with the chameleon identity.}~\cite{chameleon} is an $\numplayers$-player \textit{hidden-identity} board game. Without loss of generality, we assume $\numplayers \geq 3$. The game is played between a chameleon and \(\numplayers - 1\) non-chameleons. The goal of the non-chameleons is to identify the chameleon correctly while ensuring the integrity of a common secret jointly known by them. As we further explain later, The Chameleon requires players to strategically reveal, conceal, and infer information. We use The Chameleon as an example to measure how suitable large language models are to environments where they interact with agents of unknown identities and intentions.

\subsection{Rules of The Chameleon and Notation}

We explain the rules of The Chameleon below and provide an example gameplay in Figure \ref{fig:gameplayexample}.

\textbf{\textit{Playing order:}} The players uniformly randomly draw a playing order $\order$. Without loss of generality, we assume that each player's index is the same as its playing order, i.e., $o = (1, \ldots, \numplayers)$.

\textbf{\textit{Identity assignment:}} There is a single \textit{chameleon} uniformly randomly chosen from the players, and there are $\numplayers-1$ \textit{non-chameleon} players. The chameleon is chosen secretly: The chameleon knows its identity, but the non-chameleons do not know who the chameleon is and who the other non-chameleons are. We denote the chameleon's index with $\chameleon$.

\textbf{\textit{Category, words, and secret word:}} The players choose a category $\category$. Each category $\category$ contains $\numwordspercategory$ words denoted with $\words$. Every player knows the category $\category$ and the words $\words$. Without loss of generality, we assume these words are ordered, i.e.,  $\words = \lbrace \word_{1}, \ldots, \word_{\numwordspercategory}\rbrace$.  One of the words is chosen to be the \textit{secret word} $\secretword$ uniformly randomly from $\words$. The secret word is distributed to the non-chameleons secretly: Each non-chameleon knows the secret word, but the chameleon does not know the secret word. 

\textbf{\textit{Gameplay:}} The players play sequentially. At turn $i$, player $i$ says a \textit{response word} $\response_{i}$ that can be chosen from a (potentially infinite) set of words that is denoted by $\responses$. Player $i$ knows $\response_{1}, \ldots, \response_{i-1}$, before saying $\response_{i}$. 

\begin{figure}[t]
  \centering
\includegraphics[width=0.34\textwidth]{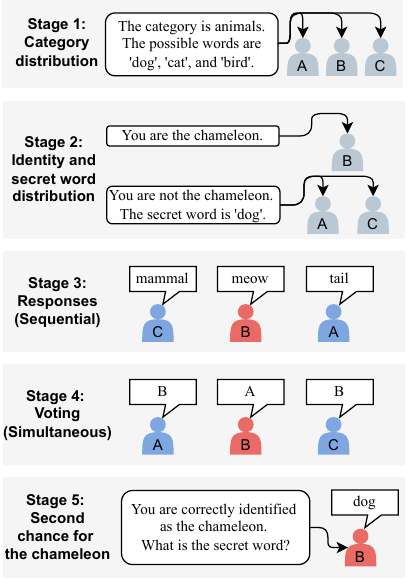}
  \caption{An example gameplay. In this example, the non-chameleons (blue players) correctly identify the chameleon (red player), but the chameleon wins in the second chance.}
  \label{fig:gameplayexample}
\end{figure}

\textbf{\textit{Discussion, voting, and a second chance:}} After all players give responses, including the chameleon, the players first discuss who the chameleon is. After the discussion, they \textit{vote} to identify the chameleon. The player with the highest vote is identified as the chameleon. If there is a tie, a uniformly randomly chosen player identifies a chameleon among the players with the highest votes. 

\begin{remark}
  In this paper, we do not consider the discussion period as a part of the game and make players vote at the same time without seeing each other's votes.  
\end{remark}
\begin{remark}
While the potential response words $\responses$ and possible secret words $\words$ depend on the category, we do not use $C$ in our notation for $\responses$ and $\words$ since the category is known to everyone. 
\end{remark}
If the non-chameleons correctly identify the chameleon, the chameleon is given a second chance. The chameleon has a single response to identify the secret word $\secretword$. The chameleon knows the category $\category$, the potential secret words $\words$, and all spoken words $\response_{1}, \ldots, \response_{\numplayers}$, before responding.

\textbf{\textit{Winning conditions:}} The game is a \textit{binary zero-sum game}. Either (i) the chameleon \textit{wins}, and the non-chameleons \textit{lose}, $(ii)$ or the non-chameleons win and the chameleon loses. 

The chameleon has two chances of winning the game. If the players misidentify the chameleon, the chameleon wins the game. If the players correctly identify the chameleon, the chameleon wins if its response in the second chance is the secret word\footnote{The scoring guide for the official game gives less points to the chameleon if it wins in the second chance. In this paper, for simplicity of analysis, we consider that the game has only two outcomes: win or loss. We also instruct the LLMs to play the game according to win or loss outcomes.}. 

The non-chameleons win the game if and only if the players correctly identify the chameleon, and the chameleon incorrectly guesses the secret word in the second chance.

\subsection{What Makes The Chameleon Challenging?}
The Chameleon is a social deduction game where the players need to reveal, conceal, and infer information. 

The non-chameleon players need to say response words that do not reveal the secret word; otherwise, the chameleon can infer the secret word. At the same time, their response words cannot be too concealing. Otherwise, they cannot recognize each other and correctly identify the chameleon. At one extreme, they can directly say the secret word not to be misidentified, but it may result in the chameleon inferring the secret word. At the other extreme, they can give a null response revealing no information to the chameleon, but it may result in a potential misidentification. In addition to revealing and concealing information during the gameplay, they need to infer who the chameleon is based on the responses. 

The chameleon needs to blend in by saying a response word that aligns with the previous players' responses. On the other hand, if the chameleon's response does not reveal any additional information to the previous players' responses, then the non-chameleon players can identify the chameleon. At one extreme, the chameleon could repeat a previous response to take no risks, but the non-chameleons may suspect the chameleon due to lack of revealing. At the other extreme, the chameleon could double down based on its estimate of the secret word, but the non-chameleons may identify the chameleon if the estimate is wrong. In addition to giving a strategic response, the chameleon needs to infer the secret word to vote for a suspicious non-chameleon and to win the game on the second chance if it gets identified.

%% file: paper_sections/theory.tex
\section{Theoretical Analysis for the Non-Chameleon Strategies}
A strategy is a player's decision-making procedure for giving a response and voting. In this section, we consider that the players can design and share strategies before the game and play the strategy during the gameplay. We consider different classes of strategies and give upper and lower bounds on the non-chameleons' win probability. We summarize the theoretical results in Figure \ref{fig:theoryresults}. The proofs for all technical results are given in \supptext{Appendix \ref{sec:app:proofs}}{the appendix}.

 Throughout this section, we assume that players pre-agree on a (possibly history-dependent) non-chameleon strategy $\pi^{\textrm{non}}$ before the game. After the identities are drawn, each non-chameleon player follows the $\pi^{\textrm{non}}$ during the game. The chameleon knows the strategy $\pi^{\textrm{non}}$ of the non-chameleons but lacks the information of the secret word. The chameleon plays strategy $\pi^{\textrm{ch}}$ during the gameplay, and $\pi^{\textrm{ch}}$ is decided knowing $\pi^{\textrm{non}}$. The chameleon's strategy $\pi^{\textrm{ch}}$ is not known by the non-chameleons while they decide on $\pi^{\textrm{non}}$.  We consider that the category $\category$ is fixed as it is known by all players and plays no role in our analysis.

 \begin{remark}
We note that the original game does not have a strategy agreement phase. We base our theoretical analysis on a shared strategy because, without a pre-agreement, the players could play arbitrary strategies. Alternatively, one could assume that the players are rational and they will play the optimal strategy by reasoning about what the others would do. Even if this strategy could be found by the LLMs (as we later show in our experiments, it is not the case), it may not be unique, and the players may face equilibrium selection issues. The shared strategy enables a meaningful theoretical analysis without relying on additional assumptions about how players choose their strategies. In addition, the instances of the same LLM have the same knowledge and reasoning, similar to a shared strategy used in our analysis.     
 \end{remark}

\begin{figure}[t]
  \centering
\includegraphics[width=0.43\textwidth]{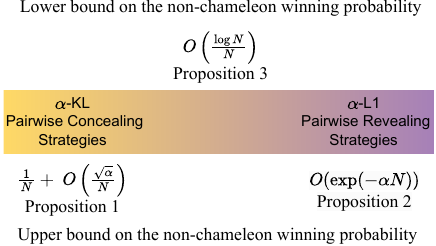}
  \caption{Bounds on the winning probability of non-chameleons. The non-chameleons lose the game with high probability if they use revealing or concealing strategies: The chameleon correctly identifies the secret word for revealing strategies, and the non-chameleons misidentify the chameleon for concealing strategies. The non-chameleons can win the game with a probability that is $\mathcal{O}(\log(\numplayers))$ times than the trivial $0$-KL pairwise concealing strategy.}
  \label{fig:theoryresults}
\end{figure}

\subsection{Converse Results for Revealing and Concealing Stationary Strategies}
\label{sec:converseresults}
We consider two stationary strategies (strategies that do not take the previous responses into account) for the non-chameleon players. Under the stationary strategy conditioned on the secret word $\word$, a non-chameleon player gives a response $r$ drawn from the distribution $\distribution^{\word}$ over $\responses$.

We first consider the following strategy that preserves ambiguity for all secret words $\words$ and potential responses $\responses$. 
\begin{definition}
    Let $KL(\distribution^{\word_{i}} || \distribution^{\word_{j}})$ be the KL divergence between distributions $\distribution^{\word_{i}}$ and $\distribution^{\word_{j}}$. A stationary non-chameleon strategy $\pi^{non}$ is $\alpha$-KL pairwise concealing if 
    \[KL(\distribution^{\word_{i}} || \distribution^{\word_{j}}) \leq \alpha \text{ for all } \word_{i}, \word_{j} \in \words.\]
\end{definition}

A trivial non-chameleon strategy is the stationary $0$-KL pairwise concealing strategy that has the same distribution for all words and votes for player $1$ regardless of the responses. One can easily observe that the winning probability for the non-chameleons is  $\nicefrac{(\numwordspercategory-1)}{\numplayers \numwordspercategory}$ as the first player is the chameleon with probability $\nicefrac{1}{N}$, and the chameleon can guess the secret word with probability $\nicefrac{1}{K}$.

Proposition \ref{prop:negklconcealing} shows that non-chameleon players lose the game with a high probability for $\alpha$-KL pairwise concealing strategies since the non-chameleons cannot statistically distinguish the chameleon at the voting stage. 
\begin{proposition}\label{prop:negklconcealing}
    For every \(\alpha\)-KL pairwise concealing non-chameleon strategy $\pi^{non}$, there exists a chameleon strategy $\pi^{\textrm{ch}}$ such that \[\Pr(\text{Non-chameleons win}) \leq \frac{1}{\numplayers} + \frac{\numplayers-1}{\numplayers^2} \sqrt{ \frac{(\numwordspercategory - 1) \alpha}{\numwordspercategory}}.\]
\end{proposition}

While $\alpha$-KL pairwise concealing strategies fail at the voting stage, Proposition \ref{prop:negklconcealinghide}, shows that under these strategies, one cannot estimate the secret word with high probability only based on the word responses given by the non-chameleon players.
\begin{proposition} \label{prop:negklconcealinghide}
 Let $\pi^{non}$ be an \(\alpha\)-KL pairwise concealing non-chameleon strategy and $\mathcal{A}$ be an arbitrary (randomized) estimator that (potentially) depends on $\pi^{non}$ and maps $\responses^{\numplayers-1}$ to $\words$. We have \[ \mathbb{E}\left[ \mathcal{A}(\mathbf{r}) = \secretword \right] \leq  \frac{1}{\numwordspercategory} + \frac{\numwordspercategory-1}{\numwordspercategory} \sqrt{(\numplayers - 1) \alpha} \] where $\mathbf{r} = [\response^{i}]_{1\leq i \leq N, i \neq c }$ is the non-chameleons' responses.
\end{proposition}

We remark that the estimator in the proposition can depend on $\pi^{non}$, but only utilizes the word responses. Additionally utilizing the voting responses given by the non-chameleons would increase the probability given in Proposition \ref{prop:negklconcealinghide}.

We next consider revealing strategies.
\begin{definition}
    Let $L1(\distribution^{w_{i}}, \distribution^{w_{j}})$ be the $L1$-distance between distributions $\distribution^{\word_{i}}$ and $\distribution^{\word_{j}}$. A stationary non-chameleon strategy $\pi^{non}$ is $\alpha$-L1 pairwise revealing if 
    \begin{equation*} L1(\distribution^{w_{i}}, \distribution^{w_{j}})
\geq \alpha  \text{ for all } i\neq j \in [\numwordspercategory].
    \end{equation*}
\end{definition}

An L1 pairwise revealing strategy means that for every pair of potential secret words, there exists a grouping of responses that distinguishes the secret word with high probability. For these types of non-chameleon strategies, the chameleon can compare the potential secret words pairwise based on the grouping and output the word that closely matches the empirical distribution of the grouped responses. We show that non-chameleons lose the game with high probability, as these strategies lead to a correct guess of the secret word, and the chameleon wins the game in the second round even if it gets voted. 

\begin{proposition} \label{prop:negl1revealing}

Let $R_{\word_{i} > \word_{j}}$ be the set of words for which $\distribution^{\word_{i}}$ is more likely than $\distribution^{\word_{i}}$. $R_{\word_{i} < \word_{j}}$ and $R_{\word_{i} = \word_{j}}$ are defined analogously. Let $p^{i}_{w^{i} >  w^{j}}$, $p^{i}_{w^{i} =  w^{j}}$, and $p^{i}_{w^{i} <  w^{j} }$ be the probability of sets $R_{\word_{i} > \word_{j}}$, $R_{\word_{i} = \word_{j}}$, and $R_{\word_{i} < \word_{j}}$, respectively, under $\distribution^{\word_{i}}$.
    Define 
    \begin{equation*}
        \resizebox{\linewidth}{!}{
        $\theta(w^{i}, w^{j}) = \max_{A \subseteq \lbrace p^{i}_{w^{i} >  w^{j}}, p^{i}_{w^{i} =  w^{j}}, p^{i}_{w^{i} <  w^{j} }\rbrace } \min \left(\sum_{a\in A} a, 1-\sum_{a\in A} a\right)$,
        }
    \end{equation*}
    \(\varphi(w^{i}, w^{j}) = \frac{1}{1 - 2\theta(w^{i}, w^{j})} \log \frac{1 -  \theta(w^{i}, w^{j})}{\theta(w^{i}, w^{j})},\) and     $\varphi= \min_{w' \in \words} \varphi(w^{*}, w')$.    
    
    For every \(\alpha\)-L1 pairwise revealing non-chameleon strategy $\pi^{non}$, there exists a chameleon $\pi^{\textrm{ch}}$ strategy such that \[\Pr(\text{Non-chameleons win}) \leq 6(\numwordspercategory-1) \exp\left(-\frac{(\numplayers - 1) \alpha \varphi}{2} \right).\]
\end{proposition}

For $\alpha \geq 1$, the bound simplifies to \[ \Pr(\text{Non-chameleons win}) \leq 6(\numwordspercategory-1) \left(\frac{2 - \alpha}{\alpha}\right)^{\frac{(N-1) \alpha}{2 \alpha -2} }.\]

A $2$-L1 pairwise revealing strategy has unique responses for each secret word. Under such a strategy, the chameleon can infer the secret word with probability $1$ after any response from the non-chameleons. However, it gets correctly identified with probability at least $\nicefrac{(K-1)}{N K}$ since it responds first with probability $\nicefrac{1}{N}$ and the response of the chameleon corresponds to the correct word with probability at most $\nicefrac{1}{K}$.

\input{tables/matchups}
\subsection{Achievability Result for a History-Dependent Strategy}

 In this section, we describe a history-dependent non-chameleon strategy that eliminates some words in each turn and maintains uniform ambiguity on the remaining words. This strategy improves the winning probability of the $0$-KL pairwise concealing strategy by a factor of \(\log (\numplayers)\).

The players create the set $\responses=2^{\words} \cup \{null\}$ of responses. Each response corresponds to a subset of the potential words. For simplicity, we assume $K = l (\numplayers + 1)$ where $l$ is a positive integer.
Let $r_{0} = \words$. We define the strategy $\pi^{amb}$ as follows:
\begin{itemize}
\setlength\itemsep{0em} 
    \item At turn $i$, if $\secretword \not\in r_{i - 1}$ or $|r_{i - 1}| \not = l(\numplayers+2-i)$, respond with $null$. If $\secretword \in r_{i - 1}$, respond with $r_{i}$ such that $r_{i}$ is uniformly randomly sampled from $\lbrace r | \secretword \in r_{i}, |r_{i}| = l (\numplayers + 1-i)\rbrace$. 
    \item At voting, if $\secretword \in r_{i}$ and $|r_{i} |  = l (N+1-i) $ for all $i$, vote for player $1$. Otherwise, vote for player $i$ such that $\secretword \not \in r_{i}$ or $|r_{i} |  \not = l (N+1-i) $ for the smallest $i$.
\end{itemize}

\begin{figure}[t]
  \centering
\includegraphics[width=0.43\textwidth]{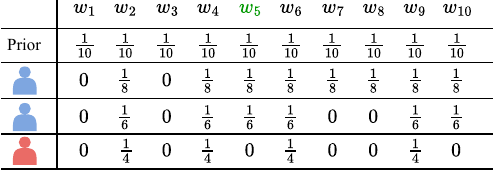}
  \caption{An example gameplay under $\pi^{amb}$ including the chameleon's response. The table shows the posterior probabilities of potential secret words given the responses under $\pi^{amb}$ without the knowledge of the secret word. The chameleon (red player) gives a response that eliminates the secret word $w_{5}$ and is not consistent with $\pi^{amb}$. Consequently, the non-chameleons (blue players) certainly identify the chameleon. The chameleon has a chance of winning the game with probability $1/2$ as it knows that the secret word is a word that it eliminated, $w_{5}$ or $w_{10}$.}
  \label{fig:lowerboundstrategy}
\end{figure}
The response generated by the strategy $\pi^{amb}$ eliminates a constant fraction $\nicefrac{1}{(\numplayers + 1)}$ of words at every turn. The remaining words contain the secret word, and each of the words has the same probability of being the secret word. Note that at turn $i$, the chameleon also has to give a response from $\lbrace r | \secretword \in r_{i}, |r_{i}| = l (\numplayers + 1 -i) \rbrace$. Since all non-chameleons initially give responses containing the secret word, the first player to respond that does not contain the secret word is the chameleon. In this case, the next non-chameleons give a null response to preserve the ambiguity on the remaining words, and the players vote for the true chameleon. If all players, including the chameleon, give responses that contain the secret word, the players vote for player 1 since all players have the same $\nicefrac{1}{\numplayers}$ probability of being the chameleon after the responses are given. An example gameplay under $\pi^{amb}$ is shown in Figure \ref{fig:lowerboundstrategy}. The non-chameleon strategy $\pi^{amb}$ wins the game with probability $\mathcal{O}(\nicefrac{\log(N)}{N})$.

\begin{proposition} \label{prop:poseliminate}
    For every chameleon strategy $\pi^{ch}$, the non-chameleon strategy $\pi^{amb}$ satisfies
    \[\Pr(\text{Non-chameleons win}) \geq \frac{\log(\numplayers + 1) + 0.4}{\numplayers} \frac{l - 1}{l}.\]
\end{proposition}

%% file: tables/matchups.tex
\begin{table*}[t]
    \centering
    \caption{Numerical results from one hundred games of The Chameleon with players using nominal LLMs.}
    \vspace{-0.2cm}
    \ra{1.08}
    \begin{tabular}{lp{2.1cm}p{1.8cm}p{2cm}|p{2.1cm}p{1.8cm}p{2cm}}
    \toprule
        \multicolumn{1}{c}{\multirow{3}{*}[-2pt]{Non-chameleon LLM}}
        & \multicolumn{3}{c}{\textsc{Chameleon: Identical LLM}}
        &\multicolumn{3}{c}{\textsc{Chameleon: GPT-5}}\\
    \cmidrule{2-4}\cmidrule{5-7}
        &\centering Non-chameleon win ratio 
        &\centering Identification ratio
        &\centering Second-chance win ratio
        &\centering Non-chameleon win ratio 
        &\centering Identification ratio
        &\centering\arraybackslash{Second-chance win ratio}\\
    \midrule
        GPT-5
            &\centering 0.00
            &\centering 0.64
            &\centering\arraybackslash{1.00}
            &\centering 0.00
            &\centering 0.64
            &\centering\arraybackslash{1.00}
        \\
        GPT-4.1
            &\centering 0.02
            &\centering 0.58
            &\centering\arraybackslash{0.96}
            &\centering 0.03
            &\centering 0.45
            &\centering\arraybackslash{0.93}
        \\
        GPT-4o
            &\centering 0.06
            &\centering 0.66
            &\centering\arraybackslash{0.91}
            &\centering 0.01
            &\centering 0.43
            &\centering\arraybackslash{0.98}
        \\
        Gemini 2.5 Pro
            &\centering 0.04
            &\centering 0.57
            &\centering\arraybackslash{0.92}
            &\centering 0.06
            &\centering 0.35
            &\centering\arraybackslash{0.83}
        \\
        Llama 3.1 70B
            &\centering 0.06
            &\centering 0.44
            &\centering\arraybackslash{0.86}
            &\centering 0.05
            &\centering 0.29
            &\centering\arraybackslash{0.82}
        \\
        Qwen3 32B
            &\centering 0.11
            &\centering 0.38
            &\centering\arraybackslash{0.71}
            &\centering 0.02
            &\centering 0.23
            &\centering\arraybackslash{0.91}
        \\
    \bottomrule
    \end{tabular}
    \label{tab:matchups}
\end{table*}

%% file: paper_sections/experiments.tex
\section{The Chameleon Experiments with LLM Agents and Discussion} \label{sec:experiments}

\subsection{Experiment Setting}
For the experiments, we use the original categories and secret words from The Chameleon, combined with additional categories and secret words generated by prompting GPT-4o (details are provided in \supptext{Appendix~\ref{sec:app:synthetic_game_cards}}{the appendix}). There are $20$ original and $54$ LLM-generated categories, and each category has $16$ potential secret words. 
Unless stated otherwise, the number $\numplayers$ of players is $4$.

We make different types of LLMs match and play the game. All non-chameleons are the same type of LLM agent, and the chameleon is a (potentially) different type of LLM agent. We use four closed-source LLMs, GPT-5~\cite{openai_gpt5}, GPT-4.1~\cite{openai_gpt41}, GPT-4o~\cite{openai_gpt4_o}, Gemini 2.5 Pro~\cite{gemini}, and two open-weight LLMs, Llama 3.1 70B~\cite{llama3}, and Qwen3 32B~\cite{qwen3}. We make the LLMs play the game $100$ times with randomly drawn categories, secret words, and response order. For each game, we instantiate a new conversation for each player.

We instruct each LLM agent on the rules of the game through text and feed responses of the agents to each other after parsing and formatting. An example conversation history and prompts are given in \supptext{Appendix \ref{sec:app:additional_experimental_details}}{the appendix}. We give the code in \url{https://github.com/mustafakarabag/LLMChameleon}. 

We present experiments with the nominal models, nominal models with additional instructions, nominal models with changed temperature, and steered models. Additional experiments using older models and inference of the secret word using a non-LLM-based method can be found in \supptext{Appendix \ref{sec:app:additional_experiments}}{the appendix}.

\subsection{Results with Nominal Models}

We report the results for the games with nominal models in Table \ref{tab:matchups}. Second-chance win ratio refers to the ratio of games where the chameleon was identified, but it won the game by correctly guessing the secret word.

\textbf{\textit{Identification ratios of the chameleon:}} 
Non-chameleon LLM agents correctly identify the chameleon LLM agent with ratios that are significantly higher than the $25\%$ ratio of the $0$-KL pairwise concealing strategies. As non-chameleons, GPT-5 agents have the best identification rate, while Qwen3 and Llama 3.1 agents have the worst identification rates.

\textbf{\textit{Win ratios of the non-chameleons:}} 
Non-chameleon LLM agents win the game with less than $6\%$ ratio (excluding the games with Qwen3 non-chameleons against Qwen3 chameleons). We note that this ratio is significantly lower than the trivially achievable $23\%$ ratio of the $0$-KL pairwise concealing strategies that do not utilize the secret word information. We remark that the low winning ratios are a result of the chameleon correctly guessing the secret word in the second chance: The chameleons correctly guess the secret word with more than $82\%$ ratio on average after being identified.

\textbf{\textit{Inference of the chameleon and the secret word:}} We observe that the non-chameleon LLM agents can infer who the chameleon is with a non-trivial probability. On the flip side, the chameleon LLM agents can infer the secret word from the non-chameleons' responses. The improved inference capabilities of LLMs (GPT-5 vs. GPT-4.1, GPT-5 vs. GPT-4o) are prominent when considering the second-chance win ratio. After being identified, chameleon GPT-5 agents can infer the secret word with a significantly higher ratio than other chameleon GPT agents. 

\input{tables/matchups_by_number_of_players}

\input{tables/matchups_mod}

\textbf{\textit{Concealment of the secret word and the chameleon identity:}}
Our theoretical analyses of concealing and revealing strategies for the non-chameleons in Section \ref{sec:converseresults} show that extremes of both strategies lose the game: The revealing strategies make the non-chameleons lose because the chameleon can infer the secret word with high probability. The concealing strategies make the non-chameleons lose because they fail to identify the chameleon. 

Combining the theoretical results with the empirical results, we deduce that non-chameleon LLM agents fail to conceal the secret word and reveal excessive information to the agents of unknown identities, making the chameleon infer the secret word easily. We note that while the non-chameleon LLM agents have the informational advantage of jointly knowing a secret word, they could achieve a higher win ratio of $23\%\ (\nicefrac{1}{4} \times \nicefrac{15}{16})$ if they were to ignore this informational advantage and play with a trivial $0$-KL pairwise concealing strategy, i.e., do not say anything informative.

Since non-chameleon LLMs use revealing strategies, it is difficult to analyze whether chameleon LLMs can blend in without being noticed. However, we note that when playing as the chameleon against another model's non-chameleons, GPT-5 has a lower identification rate than the chameleon of that model, indicating that GPT-5 can blend in better than other models as the chameleon.

\textbf{\textit{Win rates as function of the number of players:}} To demonstrate the effects of the number of players, we conducted games with $\numplayers = 3,4,6,$ and $8$ players. While the win ratios are lower than the trivially achievable levels for each value of $N$, the win and identification ratios consistently decrease as the number of players increases. This trend aligns with the theoretical results: A lower number of players yields a higher win ratio for the non-chameleons, as shown by Propositions \ref{prop:negklconcealing}, \ref{prop:negl1revealing}, and \ref{prop:poseliminate}.

\begin{figure}[t]
  \centering
\includegraphics[width=0.8\linewidth]{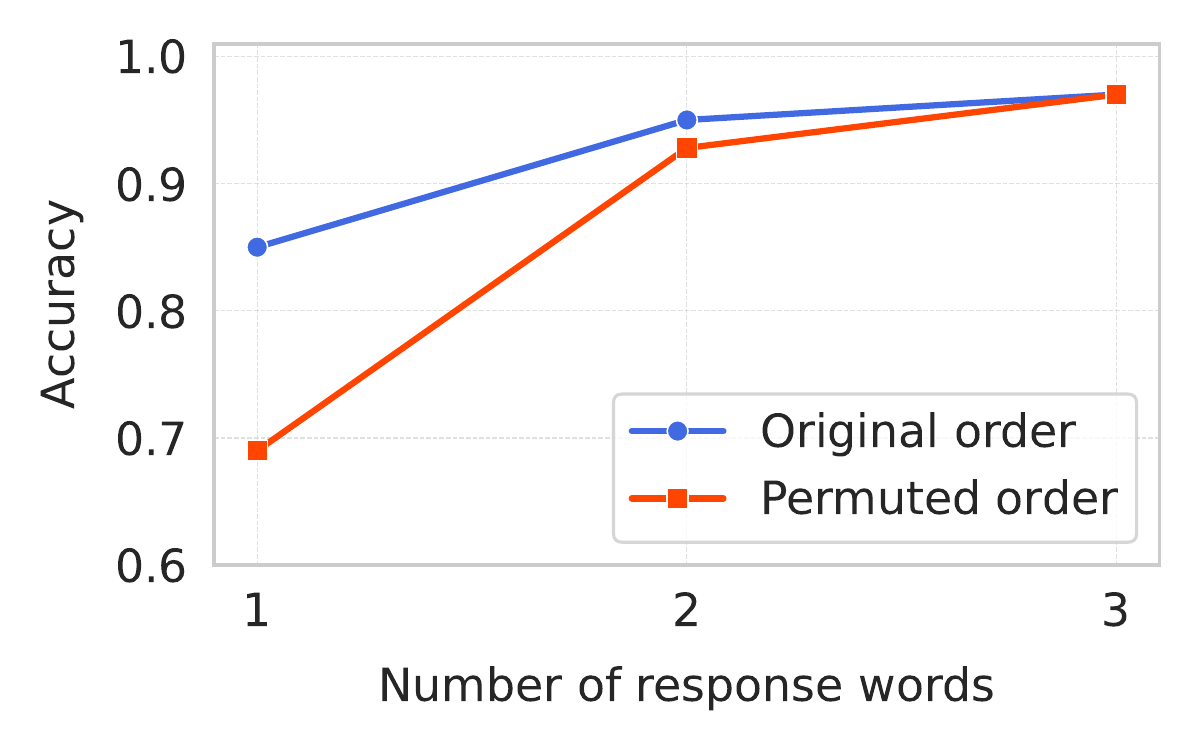}
  \caption{Accuracy of the GPT-4.1 chameleon in guessing the secret word based on the response words of GPT-4.1 non-chameleons. With the permuted order, the chameleon was presented with response words that did not match the original sequential responses of non-chameleons (\textit{e.g.}, one response word from the third instead of the first player). Evaluation done with response words from 100 games. 
}
  \label{fig:evaluator_accuracy_original_vs_permuted}
\end{figure}

\textbf{\textit{Accuracy of guessing the secret word as a function of the number and order of response words:}}
We further evaluate how individual non-chameleon LLMs contribute to the complete information revealed to the chameleon. For this, we ask the chameleon LLM to guess the secret word with only a subset of the original response words.

 As we can see in Figure \ref{fig:evaluator_accuracy_original_vs_permuted} (\textit{Original order}), the response word of the first non-chameleon already allows the chameleon to guess the secret with approximately 85\% probability, and all responses allow approximately 95\%. This evaluation further demonstrates the revealing nature of non-chameleon LLMs, as well as the inference capabilities of LLMs acting in the role of a chameleon. 

We observe that interestingly, the later responses from other non-chameleons reveal significantly less information, as can be seen from the flattening of the curve.\footnote{Note that the second and third non-chameleons see the first non-chameleon's response while generating their response words.} 
We also validate this observation by permuting the response words, e.g., by showing the chameleon the response word of the third non-chameleon instead of that of the first. This results in a significant drop in the chameleon's accuracy in guessing the secret word. The less revealing nature of later responses potentially indicates a self-reflection–driven behavior change: LLMs become aware of their suboptimal behavior and adapt a different strategy.

\subsection{Steering Based on Internal Representations}

\textbf{\textit{Linear representation of the information-revealing level in the hidden states of non-chameleon LLM:}}
To better understand the information control capabilities of LLMs, we analyze the internal representations in the intermediate layers of their Transformer architecture~\cite{vaswani2017attention}. Our analysis focuses on the open-weight LLMs Llama 3.1 70B and Qwen3 32B, as closed-source models like GPT-4 provide users only with the final outputs.

We begin by identifying and mapping out the subspace of hidden states that corresponds to the amount of information revealed by the non-chameleon LLM when it generates its response word. The procedure is as follows: First, we collect hidden states from the middle layers of the non-chameleon LLMs. We focus on these layers because previous work has found them to contain most of the high-level semantic information~\cite{geva2021transformer,logitlens,skean2025layer}. We gather these states from 700 games in which the models are instructed to adopt a specific information-revealing level and are at the point of generating their response. Second, we apply Principal Component Analysis (PCA) to these pre-collected hidden states and project them onto the first two principal components.

As shown in Figure \ref{fig:hidden_state_pca} (top), this creates a clear linear \textit{information-revealing direction} in the hidden state space of the LLMs, as captured by the first principal component. This raises the question: \textit{What is the default information-revealing level of the non-chameleon LLM when it is not given a specific instruction?}

To answer this, we remove the instruction and project the resulting hidden states onto the previously identified principal component vector. Figure \ref{fig:hidden_state_pca} (bottom) shows that in this standard gameplay (corresponding to our experiments in Table \ref{tab:matchups}), the LLM defaults to a medium information-revealing level of 6 (\textit{Steering strength 0}). This suggests that the failure of LLMs is due to their misperception of others' skills: while LLMs believe they strike a balance between revealing and concealing, others can, in fact, infer significantly more than what was intended.

\begin{figure}[t]
  \centering
\includegraphics[width=0.93\linewidth]{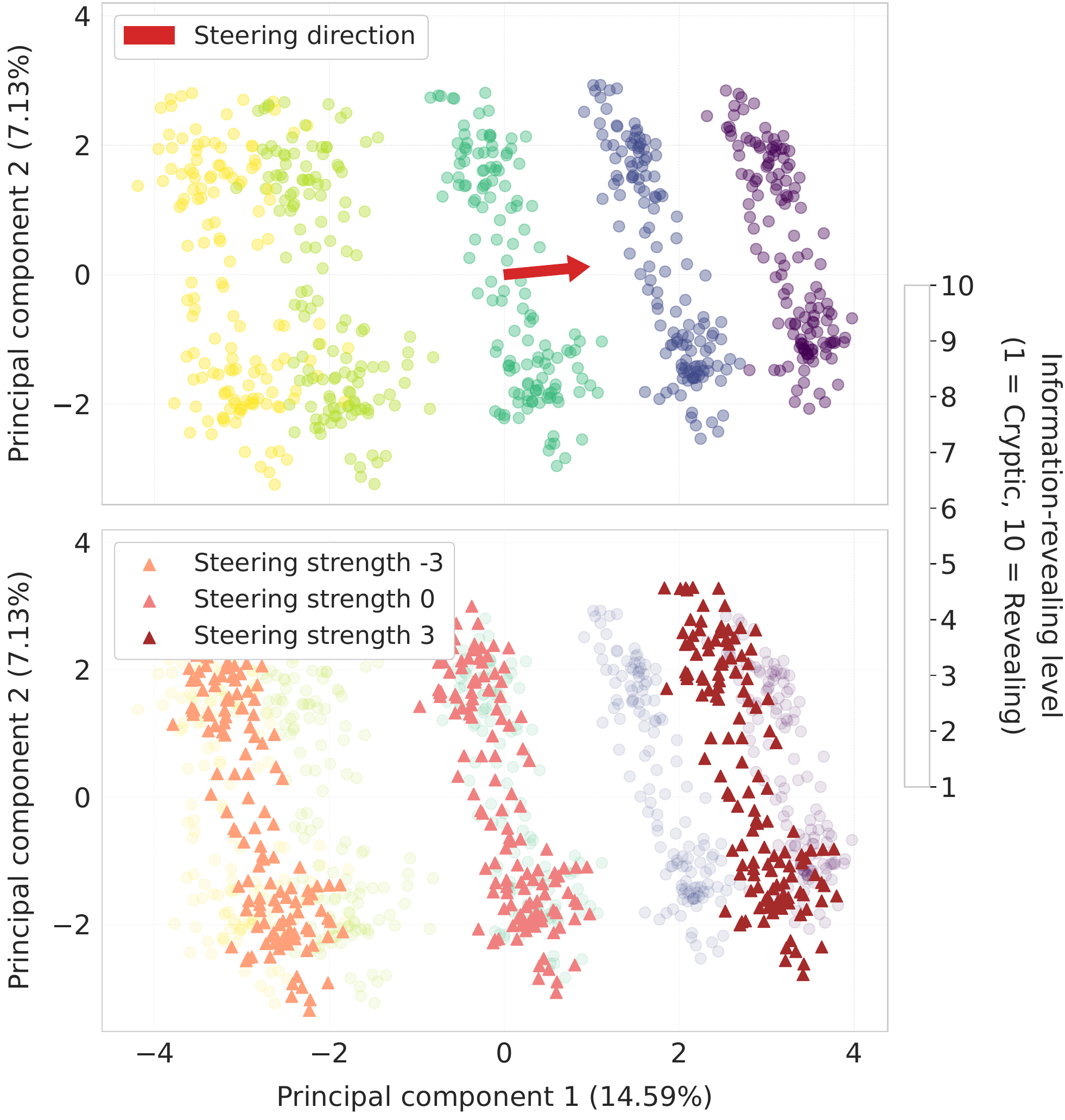}
  \caption{Principal component analysis of the hidden states of Llama 3.1 70B non-chameleon agent. \textit{Upper:} Hidden states when the LLM agent is instructed to be at a specific information-revealing level. The steering vector is rescaled and shifted for illustration purposes. \textit{Lower:} Hidden states of the LLM in a standard gameplay with no instructions (steering strength 0) and in a gameplay with no instructions but steered hidden states (steering strengths 3 and -3).}
  \label{fig:hidden_state_pca}
\end{figure}

\textbf{\textit{Hidden state steering:}}
To confirm that we identified a direction in the hidden state space that the LLM actually uses to select its response word based on the word's revealed information, we perform causal interventions. Specifically, we compute a \textit{steering vector} as the difference between the means of hidden states obtained by instructing the LLM to adopt information-revealing level 3 and 9, followed by normalization and multiplication by a scalar \textit{steering strength} (Figure \ref{fig:hidden_state_pca}, top)\footnote{To demonstrate the practicality of this method, we use hidden states collected from just four games and a single category, instead of from all 700 games used in the initial analysis. Prompts and additional details are provided in \supptext{Appendix \ref{sec:app:additional_experimental_details}.}{the appendix}.} Subsequently, we add this steering vector to the hidden state of the non-chameleon LLM when it is generating its response word. This operation, which is a variation of the Contrastive Activation Addition~\cite{rimsky2024steering}, is illustrated in Figure \ref{fig:hidden_state_pca} (bottom). To validate that this achieves the desired effect of making the non-chameleon more concealing, we make the LLMs play against each other another 100 games, but this time with non-chameleons being steered to conceal more.

\textbf{\textit{Effects of hidden state steering, instructions, temperature adjustments:}} As seen in Table \ref{tab:matchups}, hidden state steering decreases the second-chance win ratio by $44\%$ and $21\%$ for Llama 3.1 and Qwen3, respectively, and as a result, increases the non-chameleon win ratio. Example responses can be found in \supptext{Table \ref{tab:response_words} in Appendix \ref{sec:app:non_chameleon_responses}}{the appendix}. In comparison, verbal instructions to adopt information-revealing level 1 (i.e., be cryptic) are significantly less effective at making the non-chameleon LLMs less information-revealing, as shown by the second-chance win ratio being closer to the nominal case. Similarly, increasing the temperature for generating response words from the non-chameleon LLMs is insufficient to introduce concealment through randomness that would be effective enough to decrease the second-chance win ratio and increase the non-chameleon win ratio. Interestingly, we need to apply relatively high steering strength to change the LLM's responses toward concealment (e.g., 16 for the Llama 3.1, while the whole default spectrum can be spanned by a strength of 6). This suggests that the tested LLMs are incapable of strong strategic information-control capabilities by default, but contain latent skills for strategic gameplay that can be unleashed. Further exploration of how to fully express this capacity is beyond the scope of this work, but represents a promising avenue for future research.

%% file: tables/matchups_by_number_of_players.tex
\begin{table}[t]
    \centering
    \caption{Numerical results from one hundred games of The Chameleon with a varying number of Llama 3.1 70B players.}
    \vspace{-0.2cm}
    \ra{1.05}
    \resizebox{\linewidth}{!}{
    \begin{tabular}{p{1.6cm}p{2.1cm}p{1.8cm}p{2cm}}
    \toprule
        \centering Number of players
        &\centering Non-chameleon win ratio
        &\centering Identification ratio
        &\centering\arraybackslash{Second-chance win ratio}\\
    \midrule
        \centering 3
            &\centering 0.14
            &\centering 0.68
            &\centering\arraybackslash{0.79}
        \\
        \centering 4
            &\centering 0.06
            &\centering 0.44
            &\centering\arraybackslash{0.86}
        \\
        \centering 6
            &\centering 0.03
            &\centering 0.40 
            &\centering\arraybackslash{0.92}
        \\
                \centering 8
            &\centering 0.03
            &\centering 0.27
            &\centering\arraybackslash{0.89}
        \\
    \bottomrule
    \end{tabular}
    }
    \label{tab:matchups_by_number_of_players}
\end{table}

%% file: tables/matchups_mod.tex
\begin{table*}[t]
    \centering
    \caption{Numerical results from one hundred games of The Chameleon with player using LLMs with modifications.}
    \vspace{-0.2cm}
    \ra{1.08}
    \begin{tabular}{lp{2.1cm}p{1.8cm}p{2cm}|p{2.1cm}p{1.8cm}p{2cm}}
    \toprule
        \multicolumn{1}{c}{\multirow{3}{*}[-2pt]{Non-chameleon LLM}}
        & \multicolumn{3}{c}{\textsc{Chameleon: Identical LLM}}
        &\multicolumn{3}{c}{\textsc{Chameleon: GPT-5}}\\
    \cmidrule{2-4}\cmidrule{5-7}
        &\centering Non-chameleon win ratio 
        &\centering Identification ratio
        &\centering Second-chance win ratio
        &\centering Non-chameleon win ratio 
        &\centering Identification ratio
        &\centering\arraybackslash{Second-chance win ratio}\\
    \midrule
        Llama 3.1 70B
            &\centering 0.06
            &\centering 0.44
            &\centering\arraybackslash{0.86}
            &\centering 0.05
            &\centering 0.29
            &\centering\arraybackslash{0.82}
        \\
        \hspace{2pt} + temperature=1
            &\centering 0.04
            &\centering 0.33
            &\centering\arraybackslash{0.88}
            &\centering 0.00
            &\centering 0.21
            &\centering\arraybackslash{1.00}
        \\
        \hspace{2pt} + instructions
            &\centering 0.07
            &\centering 0.22
            &\centering\arraybackslash{0.68}
            &\centering 0.03
            &\centering 0.19
            &\centering\arraybackslash{0.84}
        \\
        \hspace{2pt} + steering
            &\centering 0.11
            &\centering 0.21
            &\centering\arraybackslash{0.48}
            &\centering 0.07
            &\centering 0.20
            &\centering\arraybackslash{0.63}
        \\
        Qwen3 32B
            &\centering 0.11
            &\centering 0.38
            &\centering\arraybackslash{0.71}
            &\centering 0.02
            &\centering 0.23
            &\centering\arraybackslash{0.91}
        \\
        \hspace{2pt} + temperature=1
            &\centering 0.10
            &\centering 0.36
            &\centering\arraybackslash{0.71}
            &\centering 0.04
            &\centering 0.26
            &\centering\arraybackslash{0.84}
        \\
        \hspace{2pt} + instructions
            &\centering 0.13
            &\centering 0.40
            &\centering\arraybackslash{0.68}
            &\centering 0.02
            &\centering 0.24
            &\centering\arraybackslash{0.91}
        \\
        \hspace{2pt} + steering
            &\centering 0.18
            &\centering 0.41
            &\centering\arraybackslash{0.56}
            &\centering 0.03
            &\centering 0.25
            &\centering\arraybackslash{0.88}
        \\
    \bottomrule
    \end{tabular}
    \label{tab:matchupsmod}
\end{table*}

%% file: paper_sections/conclusion.tex
\section{Conclusions}
LLM-based agents have become common in settings that include non-cooperative parties. We use The Chameleon game as a medium to analyze the information control capabilities of LLMs in such settings. In this game, a group of non-chameleon agents who do not know each other aim to identify the chameleon agent without revealing a secret. We theoretically analyze a spectrum, from revealing to concealing, of non-chameleon strategies. Combined with our theoretical analysis, the experimental results show that non-chameleon LLM agents have revealing strategies. While they can infer the chameleon LLM agent successfully, they lose the game as their responses reveal excessive information about the chameleon. We show that when instructions are used, the level of information concealment is linearly encoded in the LLM's internal representations, as clearly demonstrated in the PCA visualizations. Using this linear direction, we steer hidden states to achieve information concealment levels beyond what is achievable with instructions.  
Our results for The Chameleon motivate the development of LLMs that are suitable for strategic interactions by pointing to areas for improvement, such as the need to conceal information without deliberately misleading. 

\section*{Acknowledgements}
This work was supported in part by the Army Research Office under Grant No. W911NF-23-1-0317 and the Office of Naval Research under Grant No. N00014-24-1-2432.

%% file: paper_sections/appendix.tex
\section{Appendix}

\subsection{LLM-Generated Game Cards}
\label{sec:app:synthetic_game_cards}
For the LLM-generated categories and possible secret words, we prompted GPT-4o with the following prompt:
\begin{lstlisting}[basicstyle=\small,breaklines]
We are playing the game Chameleon but do not have enough game cards (categories and possible secret words).
Please suggest <NUM-OF-NEW-CARDS> new categories and associated 16 secret words for the game. Here are the existing categories and secret words:

<EXISTING-CATEGORIES-AND-SECRET-WORDS>

Please provide the <NUM-OF-NEW-CARDS> new categories and secret words in a JSON format, where each category is a key and the value is a list of secret words. Do not say anything else, just return the JSON.
\end{lstlisting}

\subsection{Additional Experimental Details}
\label{sec:app:additional_experimental_details}

\paragraph{LLM versions:} We used \texttt{gpt-4o-2024-08-06}, \texttt{gpt-4.1-2025-04-14}, and \texttt{gpt-5-2025-08-07} versions of GPT-4o, GPT-4.1, and GPT-5, respectively. For the open-source models, we used quantized Llama 3.1 70B (\texttt{hugging-quants/Meta-Llama-3.1-70B-Instruct-AWQ-INT4}) and Qwen3 32B (\texttt{Qwen/Qwen3-32B-AWQ}) from HuggingFace.

\paragraph{Generating answers with LLMs:}
For deterministic answers that would provide the most accurate hidden-state and game analysis, all but GPT-5 models generated their answers with temperature set to zero. GPT-5 and Gemini 2.5 Pro (\texttt{gemini-2.5-pro}) used chain of thought reasoning~\cite{wei2022chain} when generating a response. We set \texttt{reasoning effort} to \texttt{low} for GPT-5 and considered games as invalid if the answer and reasoning generation of the LLM exceeded 5,000 tokens. As can be seen from the ratios of valid games in Table \ref{tab:valid_matchup_ratios}, this occurred relatively rarely.

\input{tables/valid_matchup_ratios}

\paragraph{Generic game prompt:}
We provide an example gameplay from the perspective of one of the non-chameleon GPT-4.1 players below. \verb|GAME>| prefixes user prompts and \verb|PLAYER>| prefixes the responses of the LLM player.

\begin{lstlisting}[basicstyle=\small,breaklines]
GAME> You will play a game of Chameleon with 3 other players. The aim in this game is to unmask the Chameleon without giving away the secret word. If you are the Chameleon, your mission is to blend in with the other players, avoid detection, and work out the secret word. One of the players will be randomly selected as the Chameleon. The Chameleon will know its identity but not the secret word. The other players will know the secret word but will not know who the Chameleon is. Each player (including the Chameleon) takes it in turns to say one word related to the secret word. Take a few moments to think of your word. If another player says the word you were going to say, you can repeat it. Choose wisely! If your word is too obvious, the Chameleon might catch on and figure out the secret word. If it's too cryptic, people might start to think that you're the Chameleon. After all players have said their word, you will vote for who you think the Chameleon is. The player with the most votes will be revealed as the Chameleon. If the Chameleon is correctly revealed, the Chameleon has one last chance to guess the secret word. If the Chameleon guesses correctly, the Chameleon wins. If not, the other players win. If the Chameleon is not correctly revealed, the Chameleon wins. The other players lose. Just say yes or no to confirm that you understand the rules of the game. Do not say anything else than yes or no.
PLAYER> Yes
GAME> The category for this game is Historical Periods. The possible secret words are elizabethan era, bronze age, renaissance, roman empire, industrial revolution, baroque period, iron age, ancient egypt, cold war, age of enlightenment, byzantine era, great depression, roaring twenties, stone age, victorian era, middle ages. Just say yes or no to confirm that you understand the category and possible secret words. Do not say anything else.
PLAYER> Yes
GAME> Let's start the game. You are player number 3 in the running order. Wait for the identities to be chosen. Just say your player number to confirm that you understood your running order. Do not say anything else other than the number.
PLAYER> 3
GAME> You are not the Chameleon. The secret word is: industrial revolution. Wait for your turn to say your word. To confirm that you understand your identity, say yes if you are the Chameleon. Say no if you are not the Chameleon. Do not say anything else.
PLAYER> No
GAME> Here are what the previous players said. Player 1 said Factory. Player 2 said Steam. It is now your turn to say your word. If your word is too obvious, the Chameleon might catch on and figure out the secret word. If it's too cryptic, people might start to think that you're the Chameleon. Make sure that your response is a single word. What is your word?
PLAYER> Textiles
GAME> Here are the words that the next players said: Player 4 said Coal.
GAME> It is now your turn to vote for who you think the Chameleon is. Remember that the Chameleon does not know the secret word and the other players do. If you are the Chameleon, you can vote for another player to win the game. If you are not the Chameleon, you are trying to identify the Chameleon. Who do you think the Chameleon is? Only give the player number. Do not say anything else.
PLAYER> 2
\end{lstlisting}

\paragraph{Instruction to adopt certain information-revealing level:} For our analysis and steering of hidden states of non-chameleon LLMs, we replace the original prompt for response word generation with the following:
\begin{lstlisting}[basicstyle=\small,breaklines]
GAME> It is now your turn to say your word. If your word is too obvious, the Chameleon might catch on and figure out the secret word. If it's too cryptic, people might start to think that you're the Chameleon. On a scale of 1 to 10, where 1 is very cryptic and concealing, and 10 is highly information-revealing, your strategy should be around <INFORMATION-REVEALING-LEVEL>. Make sure that your response is a single word. What is your word?
\end{lstlisting}
where \verb|<INFORMATION-REVEALING-LEVEL>| is an integer between 1 and 10. We employ the same prompt with \verb|<INFORMATION-REVEALING-LEVEL>| set to 1 for the verbally instructed non-chameleon LLMs in Table \ref{tab:matchups}.

\paragraph{Data for hidden-state steering:} We randomly picked one category $C$ and four out of its 16 possible secret words to collect the LLM's hidden states from four games. The choice of category did not significantly affect the results. To then compute the steering vector, we ran each game with an instruction to adopt an information-revealing level of either 3 or 9. We steered the hidden states in layer 40 in both LLMs.

\clearpage
\subsection{Non-Chameleon Responses Before and After Steering}
\label{sec:app:non_chameleon_responses}
\paragraph{Response words of non-chameleons before and after steering:} In Table \ref{tab:response_words}, we provide examples of response words of the Llama 3.1 70B non-chameleons before and after steering. The responses are from the games reported in Table \ref{tab:matchupsmod} and follow the original playing order. 

\input{tables/response_words}

\clearpage

\subsection{Additional Experiments}
\label{sec:app:additional_experiments}

\paragraph{Computation of posterior probabilities via web search counts:}
We further validated that the non-chameleon LLM agents reveal excessive information by computing the posterior probabilities over potential words using web search counts (Google Search~\cite{googlesearch}).

Let \(r_{1}, \ldots, r_{m}\) be an the responses given by the first $m$ players. The Bayes' rule gives
\begin{align*}
    \Pr(\word | [r_{1}, \ldots, r_{m}]) = \frac{\Pr([r_{1}, \ldots, r_{m}] | \word) \Pr(\word)}{\Pr([r_{1}, \ldots, r_{m}])},
\end{align*} which implies \[    \Pr(\word | [r_{1}, \ldots, r_{m}]) \propto \Pr([r_{1}, \ldots, r_{m}] | \word)\]
since $\Pr(\word)$ and $\Pr([r_{1}, \ldots, r_{m}])$ are the same for all $\word \in \words$.
For a stationary strategy that does not take the previous response into account, we have \[\Pr([r_{1}, \ldots, r_{m}] | \word) = \prod_{i=1}^{m} \Pr(r_{i} | \word ).\]

We estimate the probability of a response given the word with the number of web search results \[    \Pr(r_{i} | \word ) \approx \frac{\#  r_{i} \wedge \word }{\# \word }\]
where $\# r_{i} \wedge \word$ is the number of results that contain word $\word$, and response $r_{i}$, and $\# \word $ is the number of results that contain word $\word$.
Overall, we have 
\[\Pr(\word | [r_{1}, \ldots, r_{m}]) \underset{\approx}{\propto} \prod_{i=1}^{m} \frac{\#  r_{i} \wedge \word }{\# \word }.\]

\begin{figure}[h]
    \centering
    \begin{subfigure}[]
    {\includegraphics[height=2in]{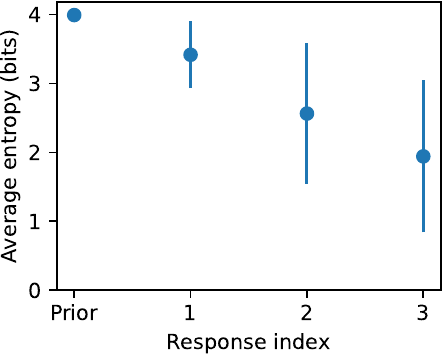}}
    \end{subfigure}%
    \quad \quad
    \begin{subfigure}[]
    {\includegraphics[height=2in]{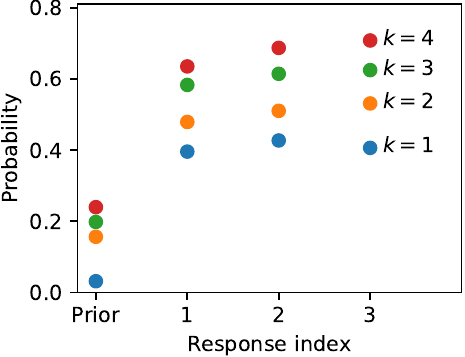}}
    \end{subfigure}
    \caption{(a) The average entropy of the posterior probabilities after responses of the non-chameleons. The prior is $4$-bits as there are $16=2^4$ potential secret words with uniform probabilities. (b) The empirical probability that the secret word is in the $k$-most likely words of the posterior probabilities. We note that the prior is not evenly distributed due to randomness over secret words.}
    \label{fig:entml}
    \end{figure}

In Figure \ref{fig:entml}, we plot the average entropy of the posterior distribution and the empirical probability that the secret word is the most likely word in the posterior distribution for the games between GPT-4o agents. We observe that the non-chameleons' responses are indeed revealing: Even after hearing a single non-chameleon response, the chameleon can guess the secret word correctly with approximately $40\%$ chance only using web search result counts. We remark that chameleon LLM agents achieve $87\%$ after hearing all responses, further indicating their inference capabilities. 

While the entropy of the posterior distribution decreases as the non-chameleons respond, interestingly, the probability of the secret word being the most likely word does not significantly increase after the first response. Aligned with the observations made in Section \ref{sec:experiments}, this phenomenon might indicate that the responses of non-chameleon LLM agents responding later may be strategic. While these agents remove some words from consideration, they do not reveal the secret word further.

\clearpage

\paragraph{Results using older models:}

In Figures \ref{fig:validgamestable}, \ref{fig:identificationtable}, and \ref{fig:wintable}, we demonstrate the identification and win ratios from games using older closed-source models. Except when they play against a GPT 3.5 Turbo chameleon, these models perform similarly to the newer models and have significantly lower win ratios than the trivially achievable levels.

\begin{figure}[h]
  \centering
\includegraphics[width=0.33\linewidth]{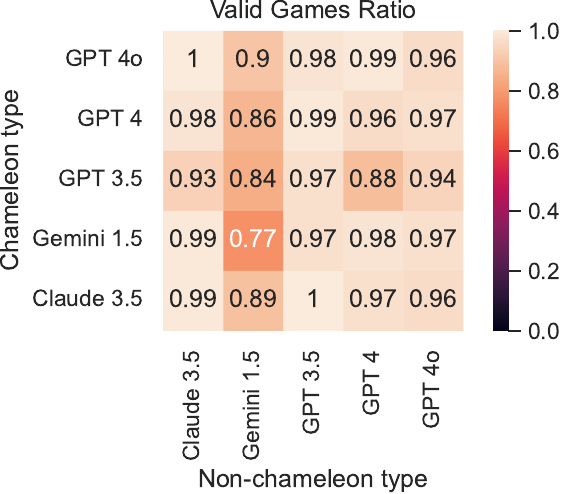}
  \caption{Ratio of valid games between LLMs.}
  \label{fig:validgamestable}
\end{figure}

\begin{figure}[h]
  \centering
\includegraphics[width=0.33\linewidth]{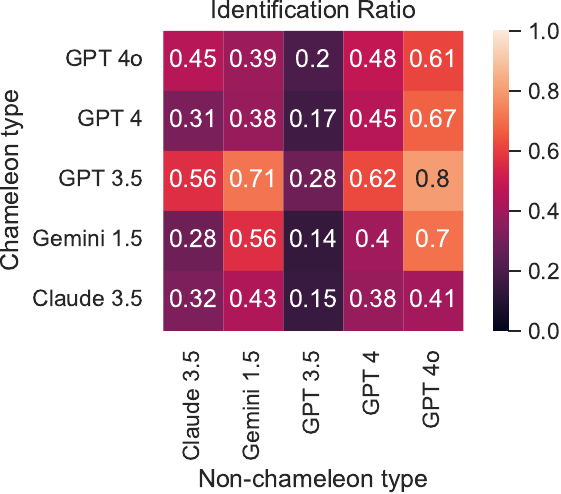}
  \caption{Ratio of games where the chameleon is correctly identified among the valid games.}
  \label{fig:identificationtable}
\end{figure}

\begin{figure}[h]
  \centering
\includegraphics[width=0.33\linewidth]{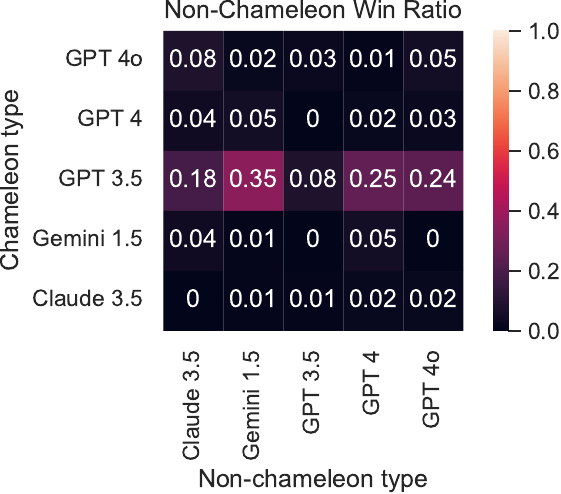}
  \caption{Ratio of games where the non-chameleons win among the valid games.}
  \label{fig:wintable}
\end{figure}

\clearpage
\subsection{Proofs for the Technical Results} 
\label{sec:app:proofs}
\input{paper_sections/proofs}

%% file: tables/valid_matchup_ratios.tex
\begin{table*}[b]
    \centering
    \caption{Valid game ratios from one hundred games of The Chameleon with LLM agents.}
    \vspace{-0.2cm}
    \ra{1.08}
    \begin{tabular}{lp{4cm}|p{4cm}}
    \toprule
        \multicolumn{1}{c}{\multirow{2}{*}[-2pt]{Non-chameleon LLM}}
        & \multicolumn{2}{c}{Valid games ratio}\\
    \cmidrule{2-3}
        &\centering \textsc{Chameleon: Identical LLM}        
        &\centering\arraybackslash{\textsc{Chameleon: GPT-5}}\\
    \midrule
        GPT-5
            &\centering 0.84
            &\centering\arraybackslash{0.84}
        \\
        GPT-4.1
            &\centering 0.98
            &\centering\arraybackslash{0.98}
        \\
        GPT-4o
            &\centering 0.99
            &\centering\arraybackslash{0.98}
        \\
        Gemini 2.5 Pro
            &\centering 0.89
            &\centering\arraybackslash{0.83}
        \\
        Llama 3.1 70B
            &\centering 0.95
            &\centering\arraybackslash{0.95}
        \\
        \hspace{2pt} + temperature=1
            &\centering 0.98
            &\centering\arraybackslash{0.98}
        \\
        \hspace{2pt} + instructions
            &\centering 0.98
            &\centering\arraybackslash{0.99}
        \\
        \hspace{2pt} + steering
            &\centering 1.00
            &\centering\arraybackslash{0.94}
        \\
        Qwen3 32B
            &\centering 1.00
            &\centering\arraybackslash{0.97}
        \\
        \hspace{2pt} + temperature=1
            &\centering 0.98
            &\centering\arraybackslash{0.97}
        \\
        \hspace{2pt} + instructions
            &\centering 0.99
            &\centering\arraybackslash{0.97}
        \\
        \hspace{2pt} + steering
            &\centering 0.99
            &\centering\arraybackslash{0.98}
        \\
    \bottomrule
    \end{tabular}
    \label{tab:valid_matchup_ratios}
\end{table*}

%% file: tables/response_words.tex
\begin{table*}[h]
    \centering
    \caption{Example response words of Llama 3.1 70B non-chameleons before and after steering.}
    \vspace{-0.2cm}
    \ra{1.08}
    \begin{tabular}{ccp{4cm}p{3cm}p{3cm}}
    \toprule
        \multicolumn{1}{c}{\multirow{2}{*}[-2pt]{Category}}
        & \multicolumn{1}{c}{\multirow{2}{*}[-2pt]{Secret word}}
        & \multicolumn{1}{c}{\multirow{2}{*}[-2pt]{Possible words}}
        & \multicolumn{2}{c}{\textsc{Response words}}\\
    \cmidrule{4-5}
        &&&\centering Before steering 
        &\centering\arraybackslash{After steering}\\
    \midrule
        Fast Food Chains & burger king
            & kfc, subway, in-n-out, chipotle, taco bell, papa john's, mcdonald's, dairy queen, arby's, little caesars, five guys, chick-fil-a, domino's, wendy's, burger king, pizza hut
            & Whopper, Bun, Patty
            & Sandwich, Fries, Vessel\\
        Historical Periods & bronze age
          & bronze age, middle ages, victorian era, baroque period, renaissance, great depression, iron age, age of enlightenment, elizabethan era, byzantine era, stone age, cold war, roaring twenties, ancient egypt, industrial revolution, roman empire
          & Copper, Tools, Mining
          & Tribes, Nomads, Migration\\
        Foods & pizza
          & pizza, taco, pancakes, noodles, fries, sushi, steak, lasagna, burger, curry, salad, dumplings, sandwich, pasta, soup, bbq
          & Cheese, Italy, Oven
          & herb, mushroom, yeast\\
        Wonders of the World & machu picchu
          & great barrier reef, statue of zeus, petra, great wall of china, mausoleum at halicarnassus, colossus of rhodes, lighthouse of alexandria, grand canyon, temple of artemis, christ the redeemer, colosseum, taj mahal, hanging gardens, chichen itza, pyramid of giza, machu picchu
          & Inca, Andes, Peru
          & Mountain, Fog, Terrain\\
        School & math
          & philosophy, gym, biology, latin, economics, art, math, spanish, physics, history, geography, technology, chemistry, music, english, religion
          & Calculator, Numbers, Equation
          & Puzzle, Code, Symbol\\
        Types of Clothing & scarf
          & tie, blouse, scarf, shorts, vest, coat, shirt, sweater, suit, dress, tunic, jacket, cardigan, skirt, trousers, jeans
          & Neck, Wool, Knit
          & Fabric, Texture, Pattern\\
        Inventions & tv
            & car, electricity, plane, wheel, telephone, radio, camera, compass, computer, printing, gunpowder, internet, tv, matches, steam engine, writing
            & Screen, Remote, Channel
            & Signal, Wave, Flicker\\
        Inventions & Radio
            & car, electricity, plane, wheel, telephone, radio, camera, compass, computer, printing, gunpowder, internet, tv, matches, steam engine, writing
            & Broadcast, Transmitter, Signal
            & Wave, Signal, Vibration\\
    \bottomrule
    \end{tabular}
    \label{tab:response_words}
\end{table*}

%% file: paper_sections/proofs.tex
\begin{proof}[Proof of Proposition \ref{prop:negklconcealing}]
    We consider that the chameleon uses the distribution $\distribution^{ch} = \frac{1}{\numwordspercategory}\sum_{i=1}^{\numwordspercategory} \distribution^{i}$ to respond and uses the non-chameleon player's strategy to vote.

    Let $(\response_{1}, \ldots, \response_{\numplayers})$ be the responses given by the players. $(\response_{1}, \ldots, \response_{\numplayers})$ follows the distribution $\boldsymbol{\distribution}^{\chameleon} = \distribution^{\secretword}\times\ldots\distribution^{\secretword}\times\distribution^{ch} \times \distribution^{\secretword}\times\ldots\distribution^{\secretword}$ where the $\chameleon-th$ component is $\distribution^{ch}$. 

    We have 
        \begin{align}
        \Pr(\text{The chameleon does not get voted}) &= \sum_{i=1}^{\numplayers} \Pr(c = i) \sum_{j\neq i}\Pr(\text{$j$ gets voted}| c=i)
        \\
        &=\sum_{i=1}^{\numplayers} \frac{1}{\numplayers} \sum_{j\neq i}\Pr(\text{$j$ gets voted}| c=i)
    \end{align}
    
    Since the outcome of the voting only depends on the responses and does not depend on the chameleon's index, we have 
    \begin{align}
        |\Pr(\text{$j$ gets voted}| c=i) - \Pr(\text{$j$ gets voted}| c=j)|  
        &\leq d_{TV}(\boldsymbol{\distribution}^{i}, \boldsymbol{\distribution}^{j})
    \end{align}
    Due to the Pinsker's inequality and the independence of the player's responses, we have
        \begin{align}
        |\Pr(\text{$j$ gets voted}| c=i) - \Pr(\text{$j$ gets voted}| c=j)|  
        &\leq d_{TV}(\boldsymbol{\distribution}^{i}, \boldsymbol{\distribution}^{j})
        \\
        &\leq \sqrt{\frac{1}{2} KL(\boldsymbol{\distribution}^{i} ||\boldsymbol{\distribution}^{j} )}
        \\
        &=
        \sqrt{\frac{1}{2} (KL(\distribution^{\secretword} || \distribution^{ch})) +  KL(\distribution^{ch} || \distribution^{\secretword}))}
    \end{align}

    By the Jensen's inequality, we have 
    \begin{align*}
        &|\Pr(\text{$j$ gets voted}| c=i) - \Pr(\text{$j$ gets voted}| c=j)| 
        \\
        &\leq
        \sqrt{\frac{1}{2} \left(\frac{1}{\numwordspercategory} \sum_{w' \in \words, w' \neq \secretword} KL(\distribution^{\secretword} || \distribution^{w'}) \right) +  \left(\frac{1}{\numwordspercategory} \sum_{w' \in \words, w' \neq \secretword} KL(\distribution^{w'} || \distribution^{\secretword}) \right)}
    \end{align*}

    Due to the $\alpha$-KL pairwise concealing property, we have 
        \begin{align} \label{eq:falsevoting}
        |\Pr(\text{$j$ gets voted}| c=i) - \Pr(\text{$j$ gets voted}| c=j)|  
        &=
        \sqrt{ \frac{(\numwordspercategory - 1) \alpha}{\numwordspercategory} }.
    \end{align}

    Since the chameleon follows the same voting strategy as the non-chameleon players, we have
    \begin{align}
        \Pr(\text{The chameleon does not get voted}) &= \sum_{i=1}^{\numplayers} \Pr(c = i) \sum_{j\neq i}\Pr(\text{$j$ gets voted}| c=i)
        \\
        &=\sum_{i=1}^{\numplayers} \frac{1}{\numplayers} \sum_{j\neq i}\Pr(\text{$j$ gets voted}| c=i)
    \end{align}
    and 
        \begin{align} \label{eq:chgetsvoted}
        \Pr(\text{The chameleon gets voted}) &=\sum_{i=1}^{\numplayers} \frac{1}{\numplayers} \Pr(\text{$i$ gets voted}| c=i).
    \end{align}

    Using \eqref{eq:falsevoting} and \eqref{eq:chgetsvoted}, we get 
        \begin{align}
        \Pr(\text{The chameleon does not get voted}) 
        &\geq \sum_{i=1}^{\numplayers} \frac{1}{\numplayers} \sum_{j\neq i}\left(\Pr(\text{$j$ gets voted}| c=j)  - \sqrt{ \frac{(\numwordspercategory - 1) \alpha}{\numwordspercategory} }\right)
        \\
        &=(\numplayers-1)\Pr(1 - \text{The chameleon gets voted}) - \frac{\numplayers-1}{\numplayers} \sqrt{ \frac{(\numwordspercategory - 1) \alpha}{\numwordspercategory} }
        \end{align}

    Using $\Pr(\text{The chameleon does not get voted}) = 1- \Pr(1 - \text{The chameleon gets voted})$, 
        \begin{align} \label{eq:chgetsvoted}
        \Pr(\text{The chameleon gets voted}) &\leq \frac{1 + \frac{\numplayers-1}{\numplayers} \sqrt{ \frac{(\numwordspercategory - 1) \alpha}{\numwordspercategory} } }{\numplayers}.
    \end{align}
    Since 
        \begin{align} \label{eq:chgetsvoted}
            \Pr(\text{Non-chameleons win}) \leq \Pr(\text{The chameleon gets voted}),
    \end{align} we get the desired result.

\end{proof}

\begin{proof}[Proof of Proposition \ref{prop:negklconcealinghide}]

 Given the secret word $\word^{j}$, the sequence of non-chameleon responses follows the distribution $\boldsymbol{\distribution}^{\word^{j}} = \distribution^{\word^{j}} \times \ldots \times\distribution^{\word^{j}}$. 
 With an abuse of notation, let $\mathcal{A}(\boldsymbol{\distribution}^{\word^{j}})$ denote the distribution of words output by the estimator $\mathcal{A}$ and $\mathcal{A}(\boldsymbol{\distribution}^{\word^{j}})_{\word^{i}}$ be the probability of $\word^{i}$ under $\mathcal{A}(\boldsymbol{\distribution}^{\word^{j}})_{\word^{i}}$. 

 We have
\begin{align}
    \mathbb{E}_{\mathbf{r}}\left[ \mathcal{A}(\mathbf{r}) \neq \secretword \right] &= \sum_{i=1}^{\numwordspercategory} \sum_{j\neq i} \Pr(\secretword = \word^{j}) \mathcal{A}(\boldsymbol{\distribution}^{\word^{j}})_{\word^{i}}
    \\
    &= \sum_{i=1}^{\numwordspercategory} \sum_{j\neq i} \frac{1}{\numwordspercategory} \mathcal{A}(\boldsymbol{\distribution}^{\word^{j}})_{\word^{i}}
    \\
    &= \sum_{i=1}^{\numwordspercategory} \sum_{j\neq i} \frac{1}{\numwordspercategory}  ( \mathcal{A}(\boldsymbol{\distribution}^{\word^{j}})_{\word^{i}} - \mathcal{A}(\boldsymbol{\distribution}^{\word^{i}})_{\word^{i}} + \mathcal{A}(\boldsymbol{\distribution}^{\word^{i}})_{\word^{i}})
    \\
    &= \sum_{i=1}^{\numwordspercategory} \sum_{j\neq i} \frac{1}{\numwordspercategory}   \mathcal{A}(\boldsymbol{\distribution}^{\word^{i}})_{\word^{i}} - \sum_{i=1}^{\numwordspercategory} \sum_{j\neq i} \frac{1}{\numwordspercategory} (\mathcal{A}(\boldsymbol{\distribution}^{\word^{i}})_{\word^{i}} - \mathcal{A}(\boldsymbol{\distribution}^{\word^{j}})_{\word^{i}})
        \\
    &= \sum_{i=1}^{\numwordspercategory} \sum_{j\neq i} \Pr(\secretword = \word^{i})    \mathcal{A}(\boldsymbol{\distribution}^{\word^{i}})_{\word^{i}} - \sum_{i=1}^{\numwordspercategory} \sum_{j\neq i} \frac{1}{\numwordspercategory} (\mathcal{A}(\boldsymbol{\distribution}^{\word^{i}})_{\word^{i}} - \mathcal{A}(\boldsymbol{\distribution}^{\word^{j}})_{\word^{i}})
    \\
    &= (\numwordspercategory-1)\mathbb{E}_{\mathbf{r}}\left[ \mathcal{A}(\mathbf{r}) = \secretword \right]- \sum_{i=1}^{\numwordspercategory} \sum_{j\neq i} \frac{1}{\numwordspercategory} (\mathcal{A}(\boldsymbol{\distribution}^{\word^{i}})_{\word^{i}} - \mathcal{A}(\boldsymbol{\distribution}^{\word^{j}})_{\word^{i}})  
        \\
    &\geq (\numwordspercategory-1)\mathbb{E}_{\mathbf{r}}\left[ \mathcal{A}(\mathbf{r}) = \secretword \right]- \sum_{i=1}^{\numwordspercategory} \sum_{j\neq i} \frac{1}{\numwordspercategory} |(\mathcal{A}(\boldsymbol{\distribution}^{\word^{i}})_{\word^{i}} - \mathcal{A}(\boldsymbol{\distribution}^{\word^{j}})_{\word^{i}})|  
        \\
    &\geq (\numwordspercategory-1)\mathbb{E}_{\mathbf{r}}\left[ \mathcal{A}(\mathbf{r}) = \secretword \right]- \sum_{i=1}^{\numwordspercategory} \sum_{j\neq i} \frac{1}{\numwordspercategory}  d_{TV}(\mathcal{A}(\boldsymbol{\distribution}^{\word^{j}}), \mathcal{A}(\boldsymbol{\distribution}^{\word^{i}}))   
\end{align}
where $d_{TV}(\mathcal{A}(\boldsymbol{\distribution}^{\word^{j}}), \mathcal{A}(\boldsymbol{\distribution}^{\word^{i}}))$ denotes the total variation distance between $\mathcal{A}(\boldsymbol{\distribution}^{\word^{j}})$ and $\mathcal{A}(\boldsymbol{\distribution}^{\word^{i}})$.
Due to the Pinsker's inequality, we have 
\begin{align}
    \mathbb{E}_{\mathbf{r}}\left[ \mathcal{A}(\mathbf{r}) \neq \secretword \right]
    &\geq (\numwordspercategory-1)\mathbb{E}_{\mathbf{r}}\left[ \mathcal{A}(\mathbf{r}) \leq \secretword \right]- \sum_{i=1}^{\numwordspercategory} \sum_{j\neq i} \frac{1}{\numwordspercategory} d_{TV}(\mathcal{A}(\boldsymbol{\distribution}^{\word^{j}}), \mathcal{A}(\boldsymbol{\distribution}^{\word^{i}}))
    \\
    & \geq (\numwordspercategory-1)\mathbb{E}_{\mathbf{r}}\left[ \mathcal{A}(\mathbf{r}) = \secretword \right]- \sum_{i=1}^{\numwordspercategory} \sum_{j\neq i} \frac{1}{\numwordspercategory} \sqrt{KL(\mathcal{A}(\boldsymbol{\distribution}^{\word^{j}})|| \mathcal{A}(\boldsymbol{\distribution}^{\word^{i}}))}
\end{align}

Due to the data processing inequality, the fact that responses are i.i.d., and $\pi^{non}$ is $\alpha$-KL pairwise concealing, we have \begin{align}
    \mathbb{E}_{\mathbf{r}}\left[ \mathcal{A}(\mathbf{r}) \neq \secretword \right]
    &\geq  (\numwordspercategory-1)\mathbb{E}_{\mathbf{r}}\left[ \mathcal{A}(\mathbf{r}) = \secretword \right]- \sum_{i=1}^{\numwordspercategory} \sum_{j\neq i} \frac{1}{\numwordspercategory} \sqrt{KL(\mathcal{A}(\boldsymbol{\distribution}^{\word^{j}})|| \mathcal{A}(\boldsymbol{\distribution}^{\word^{i}}))}
    \\
        &\geq  (\numwordspercategory-1)\mathbb{E}_{\mathbf{r}}\left[ \mathcal{A}(\mathbf{r}) = \secretword \right]- \sum_{i=1}^{\numwordspercategory} \sum_{j\neq i} \frac{1}{\numwordspercategory} \sqrt{KL(\boldsymbol{\distribution}^{\word^{j}}|| \boldsymbol{\distribution}^{\word^{i}})}
            \\
        &\geq  (\numwordspercategory-1)\mathbb{E}_{\mathbf{r}}\left[ \mathcal{A}(\mathbf{r}) = \secretword \right]- \sum_{i=1}^{\numwordspercategory} \sum_{j\neq i} \frac{1}{\numwordspercategory} \sqrt{(\numplayers-1)KL(\distribution^{\word^{j}}|| \distribution^{\word^{i}})}
                    \\
        &\geq  (\numwordspercategory-1)\mathbb{E}_{\mathbf{r}}\left[ \mathcal{A}(\mathbf{r}) = \secretword \right]- \sum_{i=1}^{\numwordspercategory} \sum_{j\neq i} \frac{1}{\numwordspercategory} \sqrt{(\numplayers-1) \alpha}
                            \\
        &=  (\numwordspercategory-1)\mathbb{E}_{\mathbf{r}}\left[ \mathcal{A}(\mathbf{r}) = \secretword \right]- (K-1) \sqrt{(\numplayers-1) \alpha}
\end{align}
By noting that $\mathbb{E}_{\mathbf{r}}\left[ \mathcal{A}(\mathbf{r}) \neq \secretword \right] = 1 - \mathbb{E}_{\mathbf{r}}\left[ \mathcal{A}(\mathbf{r}) = \secretword \right]$, we get the desired result.

\end{proof}

\begin{proof}[Proof of Proposition \ref{prop:negl1revealing}]
     We show that with high probability, the chameleon can win the game in the second stage. We first establish a connection between the L1 distance, the total variation distance, and a partition of the responses $\responses$.

We define 
$$R_{w_{i} >  w_{j}} = \lbrace r | r \in \responses, \distribution^{w_{i}}(r) > \distribution^{ w_{j}}(r) \rbrace,$$
    $$R_{w_{i} = w_{j}} = \lbrace r | r \in \responses, \distribution^{w_{i}}(r) = \distribution^{ w_{j}}(r) \rbrace,$$
    $$R_{w_{i} <  w_{j}} = \lbrace r | r \in \responses ,\distribution^{ w_{i}}(r) < \distribution^{w_{j}}(r) \rbrace$$ where $R_{w_{i} > w_{j}}$ is the set of responses for which word $i$ is more likely, $R_{w_{i} = w_{j}}$ is the set of responses for which words $i$ and $j$ are equally likely, and $R_{w_{i} < w_{j}}$ is the set of responses for which word $j$ is more likely.

    We define the probabilities of these sets under the distribution $\distribution^{w_{i}}$
    $$p^{i}_{w_{i} >  w_{j}} = \sum_{r \in R_{w_{i} >  w_{j}}} \distribution^{w_{i}}(r),$$
    $$p^{i}_{w_{i} =  w_{j}} = \sum_{r \in R_{w_{i} =  w_{j}}} \distribution^{w_{i}}(r),$$ and
    $$p^{i}_{ w_{i} < w_{j}} = \sum_{r \in R_{ w_{i} < w_{j}}} \distribution^{w_{i}}(r)$$ 

    Similarly, we define the probabilities of these sets under the distribution $\distribution^{w_{j}}$ (Note that the sets are in the reverse order.)
        $$p^{j}_{w_{j} >  w_{i}} = \sum_{r \in R_{w_{i} <  w_{j}}} \distribution^{w_{j}}(r),$$
    $$p^{j}_{w_{j} =  w_{i}} = \sum_{r \in R_{w_{i} =  w_{j}}} \distribution^{w_{j}}(r),$$ and
    $$p^{j}_{ w_{j} < w_{i}} = \sum_{r \in R_{ w_{i} < w_{j}}} \distribution^{w_{i}}(r).$$ 

Let $TV(\distribution^{w_{i}}, \distribution^{w_{j}})$ be the total variation distance between $\distribution^{w_{i}}$ and $\distribution^{w_{j}}$.
Note that $L1(\distribution^{w_{i}}, \distribution^{w_{j}}) \geq \alpha$ implies $TV(\distribution^{w_{i}}, \distribution^{w_{j}})
\geq \nicefrac{\alpha}{2}$.
     Since $p^{i}_{w_{i} =  w_{j}} = p^{j}_{ w_{j} = w_{i}}$ by definition of the total variation of distance $TV(\distribution^{w_{i}}, \distribution^{w_{j}})
\geq \nicefrac{\alpha}{2}$ we also have
 $$    \left\| \begin{bmatrix}
        p^{i}_{w_{i} >  w_{j}}
        \\
        p^{i}_{w_{i} =  w_{j}}
        \\
         p^{i}_{w_{i} < w_{j}}
    \end{bmatrix} -
         \begin{bmatrix}
        p^{j}_{ w_{j} < w_{i}}
        \\
        p^{j}_{ w_{j} = w_{i}}
        \\
         p^{j}_{ w_{j} > w_{i}}
    \end{bmatrix} \right\|_{1}  \geq \alpha$$.

     Since $| p^{i}_{w_{i} >  w_{j}} - p^{j}_{ w_{j} < w_{i}}| = | p^{i}_{w_{i} <  w_{j}} - p^{j}_{ w_{j} > w_{i}}|$ for a stationary strategy $\pi^{non}$ that is $\alpha$-L1 pairwise revealing, we have $ p^{i}_{w_{i} >  w_{j}} - p^{j}_{ w_{j} < w_{i}} \geq \nicefrac{\alpha}{2}$ for all $i\neq j \in [\numwordspercategory]$.

     Let the chameleon use an arbitrary distribution to give a response. The chameleon uses the following strategy to estimate the secret word. 
     For every pair $w\neq w' \in \words$, build a discrete probability distribution such that $$\hat{p}_{w > w'} = \sum_{r \in R_{w > w'} \cap \lbrace r_{i} | i\neq \chameleon\rbrace} \frac{1}{\numplayers-1},$$
     $$\hat{p}_{w = w'} = \sum_{r \in R_{w = w'} \cap \lbrace r_{i} | i\neq \chameleon\rbrace} \frac{1}{\numplayers-1},$$
     and $$\hat{p}_{w < w'} = \sum_{r \in R_{w < w'} \cap \lbrace r_{i} | i\neq \chameleon\rbrace} \frac{1}{\numplayers-1}.$$ In words, the chameleon builds an empirical distribution based on the others' responses and collapses the empirical distribution into bins, the words for which $w$ is more likely, the words for which $w$ and $w'$ are equally likely, and the words for which $w'$ is more likely.

     The chameleon uses the following method to estimate the secret word 
\begin{itemize}
    \item Output $w^{i}$, if 
    \begin{equation}
         \left\| \begin{bmatrix}
        p^{i}_{w^{i} > w^{j}}
        \\
        p^{i}_{w^{i} = w^{j}}
        \\
         p^{i}_{w^{i} < w^{j}}
    \end{bmatrix} -
         \begin{bmatrix}
       \hat{p}_{w^{i} > w^{j}}
        \\
        \hat{p}_{w^{i} = w^{j}}
        \\
        \hat{p}_{w^{i} < w^{j}}
    \end{bmatrix} \right\|_{1} < \frac{\alpha}{2}
    \end{equation}
    for all $w^{i} \neq w^{j} \in \words$.
    \item Output an arbitrary $w\in \words $ otherwise.
\end{itemize}
We remark that there cannot be more than a single output in the first case. Since     \begin{equation}
         \left\| \begin{bmatrix}
        p^{i}_{w^{i} > w^{j}}
        \\
        p^{i}_{w^{i} = w^{j}}
        \\
         p^{i}_{w^{i} < w^{j}}
    \end{bmatrix} -
         \begin{bmatrix}
       \hat{p}_{w^{i} > w^{j}}
        \\
        \hat{p}_{w^{i} = w^{j}}
        \\
        \hat{p}_{w^{i} < w^{j}}
    \end{bmatrix} \right\|_{1} < \frac{\alpha}{2}
    \end{equation}  and 
        \begin{equation*}
    \left\| \begin{bmatrix}
        p^{i}_{w_{i} >  w_{j}}
        \\
        p^{i}_{w_{i} =  w_{j}}
        \\
         p^{i}_{w_{i} < w_{j}}
    \end{bmatrix} -
         \begin{bmatrix}
        p^{j}_{ w_{j} < w_{i}}
        \\
        p^{j}_{ w_{j} = w_{i}}
        \\
         p^{j}_{ w_{j} > w_{i}}
    \end{bmatrix} \right\|_{1} \geq \alpha
    \end{equation*}
    we have \begin{equation}
         \left\| \begin{bmatrix}
        p^{j}_{ w_{j} < w_{i}}
        \\
        p^{j}_{ w_{j} = w_{i}}
        \\
         p^{j}_{ w_{j} > w_{i}}
    \end{bmatrix} -
         \begin{bmatrix}
       \hat{p}_{w^{i} > w^{j}}
        \\
        \hat{p}_{w^{i} = w^{j}}
        \\
        \hat{p}_{w^{i} < w^{j}}
    \end{bmatrix} \right\|_{1} > \frac{\alpha}{2}
    \end{equation}
    due to the triangle inequality. Since $\hat{p}_{w^{i} > w^{j}} = \hat{p}_{w^{j} < w^{i}}$ and $\hat{p}_{w^{i} < w^{j}} = \hat{p}_{w^{j} > w^{i}}$, it also implies 
    \begin{equation}
         \left\| \begin{bmatrix}
        p^{j}_{ w_{j} < w_{i}}
        \\
        p^{j}_{ w_{j} = w_{i}}
        \\
         p^{j}_{ w_{j} > w_{i}}
    \end{bmatrix} -
         \begin{bmatrix}
       \hat{p}_{w^{j} < w^{i}}
        \\
        \hat{p}_{w^{j} = w^{i}}
        \\
        \hat{p}_{w^{j} > w^{i}}
    \end{bmatrix} \right\|_{1} > \frac{\alpha}{2}
    \end{equation}
    which violates the condition for the first case.

    Without loss of generality, let $i$ be the index of the secret word in $\words$ note that $\lbrace r_{l} | l\neq \chameleon\rbrace$ are independently sampled from $\distribution^{\word^{i}}$ which implies that each tuple $(\mathds{1}(r_{l} \in R_{w^{i} > w^{j}}), \mathds{1}(r_{l} \in R_{w^{i} = w^{j}}), \mathds{1}(r_{l} \in R_{w^{i} < w^{j}}))$ is sampled from the categorical distribution $$\begin{bmatrix}
        p^{i}_{w^{i} > w^{j}}
        \\
        p^{i}_{w^{i} = w^{j}}
        \\
         p^{i}_{w^{i} < w^{j}}
    \end{bmatrix}.$$
    Consequently, via Theorem 2.1 of \cite{weissman2003inequalities} we have 
    \begin{equation*}
    \Pr\left(            \left\| \begin{bmatrix}
        p^{i}_{w^{i} > w^{j}}
        \\
        p^{i}_{w^{i} = w^{j}}
        \\
         p^{i}_{w^{i} < w^{j}}
    \end{bmatrix} -
         \begin{bmatrix}
       \hat{p}_{w^{i} > w^{j}}
        \\
        \hat{p}_{w^{i} = w^{j}}
        \\
        \hat{p}_{w^{i} < w^{j}}
    \end{bmatrix} \right\|_{1} < \frac{\alpha}{2}\right) \geq 1- 2 \exp\left(-\frac{(\numplayers-1)\varphi(w^{i}, w^{j})\alpha^{2}}{16}\right)
    \end{equation*}
    where \[ \theta(w^{i}, w^{j}) = \max_{A \subseteq \lbrace p^{i}_{w^{i} >  w^{j}}, p^{i}_{w^{i} >  w^{j}}, p^{i}_{w^{i} <  w^{j} }\rbrace } \min \left(\sum_{a\in A} a, 1-\sum_{a\in A} a\right),\]
    and \(\varphi = \frac{1}{1 - 2\theta(w^{i}, w^{j})} \log \frac{1 -  \theta(w^{i}, w^{j})}{\theta(w^{i}, w^{j})}\).   

    Taking the minimum over $\phi(w^{i}, w^{j})$ over $j$ and using the union bound shows that the chameleon outputs the correct word $w^{i}$ with probability at least \[1-6(\numwordspercategory-1) \exp\left(-\frac{(\numplayers - 1) \alpha \varphi}{2} \right).\]

\end{proof}

\begin{proof}[Proof of Proposition \ref{prop:poseliminate}]
We first analyze the case $|r_{\chameleon}| = k (\numplayers + 1 -\chameleon)$, i.e., the chameleon eliminates the same number of words as the non-chameleon players. Since the non-chameleons preserve uniform probability over the remaining words to be the secret word, we have  
\begin{equation}\label{eq:chfailstoguess}
    \Pr(\secretword \not \in \response_{\chameleon} | c) = \frac{1}{\numplayers+2-\chameleon}
\end{equation}
    regardless of how the chameleon chooses $l$ words to be eliminated among $l(\numplayers+2-\chameleon)$ words. The chameleon loses in two cases: (i) it does not eliminate the secret word, but its index is $1$, it gets voted and gives the wrong response in the second chance, and (ii) it eliminates the secret word, it gets voted and gives the wrong response in the second stage. 
    Consequently, we have

    \begin{align*}
        \Pr(\text{Non-chameleons win}) 
    &=\Pr(\chameleon = 1) \Pr(\secretword \in r_{1}| c=1)\Pr(\word^{ch} \not = \secretword |\secretword \in \response_{1}, \chameleon=1) 
    \\
    &+ \sum_{i=1}^{\numplayers} \Pr(\chameleon = i) \Pr(\secretword \not \in \response_{i} | \chameleon =i) \Pr(\word^{ch}| \chameleon= i, \secretword \not \in \response_{i})
    \end{align*}

    We have 
    \[\Pr(c = 1) \Pr(\secretword \in r_{1}| c=1)\Pr(\word^{ch} \not = \secretword |\secretword \in r_{1}, c=1) = \frac{1}{\numplayers} \frac{N}{N+1} \frac{l-1}{l}\] since the chameleon has index $1$ with probability $1/\numplayers$, \eqref{eq:chfailstoguess} holds for all $c$, and there are $l$ remaining words after all players respond and all have the same probability of being the secret word.

    If (ii) happens, because it got voted, the chameleon can understand that it eliminated the secret word. We have 
    \[\Pr(c = i) \Pr(\secretword \not \in r_{i} | c =i) \Pr(\word^{ch}| c= i, \secretword \not \in r_{i}) = \frac{1}{\numplayers} \frac{1}{N+2-i} \frac{l-1}{l}\] since the chameleon has index $i$ with probability $1/\numplayers$, \eqref{eq:chfailstoguess} holds for all $c=i$, and each word that the chameleon eliminated has $1/l$ probability of being the secret word.

    Using an upper bound for the harmonics series, we have 
    \[\sum_{i=1}^{\numplayers} \Pr(c = i) \Pr(\secretword \not \in r_{i} | c =i) \Pr(\word^{ch}| c= i, \secretword \not \in r_{i}) \geq \frac{\log(\numplayers+1) + \frac{1}{2\numplayers+3} + \gamma -1 }{\numplayers} \frac{l -1}{l}
    \] where $\gamma\approx 0.577$ is the Euler–Mascheroni constant.
    
    Adding up the losing probabilities for both cases and using $N\geq 3$, we get 
    \[\Pr(\text{Non-chameleons win}) \geq \frac{\log(\numplayers + 1) + 0.4}{\numplayers} \frac{l - 1}{l}. \]

    In the case, $|r_{\chameleon}| \neq k (\numplayers + 1 -\chameleon)$, i.e., the chameleon does not eliminate the same number of words as the non-chameleon players. The non-chameleons can identify the chameleon with probability $1$ regardless of the chameleon's index since the chameleon gets voted because it eliminated a wrong number of words. The probability that the chameleon guesses the secret word correctly is trivially lower than the case $|r_{\chameleon}| = k (\numplayers + 1 -\chameleon)$. Consequently, we again have \[\Pr(\text{Non-chameleons win}) \geq \frac{\log(\numplayers + 1) + 0.4}{\numplayers} \frac{l - 1}{l}. \]
\end{proof}